\documentclass[sigconf]{acmart}
\AtBeginDocument{%
  }

\copyrightyear{2025}
\acmYear{2025}
\setcopyright{cc}
\setcctype{by}
\acmConference[KDD '25]{Proceedings of the 31st ACM SIGKDD Conference on Knowledge Discovery and Data Mining V.2}{August 3--7, 2025}{Toronto, ON, Canada}
\acmBooktitle{Proceedings of the 31st ACM SIGKDD Conference on Knowledge Discovery and Data Mining V.2 (KDD '25), August 3--7, 2025, Toronto, ON, Canada}
\acmDOI{10.1145/3711896.3736969}
\acmISBN{979-8-4007-1454-2/2025/08}

\usepackage{hyperref}
\usepackage{url}

\usepackage{booktabs}       % professional-quality tables
\usepackage{amsfonts}       % blackboard math symbols
\usepackage{nicefrac}       % compact symbols for 1/2, etc.
\usepackage{microtype}      % microtypography
\usepackage{xcolor}         % colors

\usepackage{graphicx}
\usepackage{algorithm}
\usepackage{algorithmicx}
\usepackage{algpseudocode}
\usepackage{amsmath}
\usepackage{amsthm}
\usepackage{bbm}
\usepackage{textcomp}
\usepackage{multirow}
\usepackage{hyperref}
\usepackage{diagbox}
\usepackage{url}
\usepackage{enumitem}
\usepackage{makecell}
\usepackage{tabularx} 
\usepackage{subcaption}

\newtheorem{definition}{Definition}%[theorem]
\newtheorem{proposition}{Proposition}
\usepackage[normalem]{ulem}
\useunder{\uline}{\ul}{}
\usepackage[font=small]{caption}
\usepackage{tablefootnote}

\newcommand{\revise}[1]{\textcolor{black}{#1}} 

\begin{document}

\title{FreRA: A Frequency-Refined Augmentation for Contrastive Learning on Time Series Classification}

\author{Tian Tian}
\email{tian006@e.ntu.edu.sg}
\affiliation{
  \institution{Alibaba-NTU Singapore Joint Research Institute, Interdisciplinary Graduate Programme, Nanyang Technological University}
  \country{Singapore}}

\author{Chunyan Miao}
\email{ascymiao@ntu.edu.sg}
\affiliation{
  \institution{College of Computing and Data Science, Nanyang Technological University}
  \country{Singapore}}

\author{Hangwei Qian}
\email{qian_hangwei@cfar.a-star.edu.sg}
\affiliation{
  \institution{CFAR, A*STAR}
  \country{Singapore}}
  
\begin{abstract}
Contrastive learning has emerged as a competent approach for unsupervised representation learning. However, the design of an optimal augmentation strategy, although crucial for contrastive learning, is less explored for time series classification tasks. Existing predefined time-domain augmentation methods are primarily adopted from vision and are not specific to time series data. Consequently, this cross-modality incompatibility may distort the \revise{semantically relevant information} of time series by introducing mismatched patterns into the data. 
To address this limitation, we present a novel perspective from the frequency domain and identify three advantages for downstream classification: 1) the frequency component naturally encodes \textit{global} features, 2) the orthogonal nature of the Fourier basis allows \textit{easier isolation} and \textit{independent modifications} of critical and unimportant information, and 3) a \textit{compact} set of frequency components can preserve semantic integrity. To fully utilize the three properties, we propose the lightweight yet effective \textbf{Fre}quency-\textbf{R}efined \textbf{A}ugmentation (\textbf{FreRA}) tailored for time series contrastive learning on classification tasks, which can be seamlessly integrated with contrastive learning frameworks in a plug-and-play manner. Specifically, FreRA automatically separates critical and unimportant frequency components. Accordingly, we propose \revise{semantic-aware} Identity Modification and \revise{semantic-agnostic} Self-adaptive Modification to protect \revise{semantically relevant information} in the critical frequency components and infuse variance into the unimportant ones respectively. 
% The independent modifications ensure that the added variance does not affect the critical information. 
Theoretically, we prove that FreRA generates semantic-preserving views. Empirically, we conduct extensive experiments on two benchmark datasets, including UCR and UEA archives, as well as five large-scale datasets on diverse applications. FreRA consistently outperforms ten leading baselines on time series classification, anomaly detection, and transfer learning tasks, demonstrating superior capabilities in contrastive representation learning and generalization in transfer learning scenarios across diverse datasets. The code is available at https://github.com/Tian0426/FreRA.
\end{abstract}






\maketitle

\begin{figure}[t]
  \centering
  \includegraphics[width=0.46\textwidth]{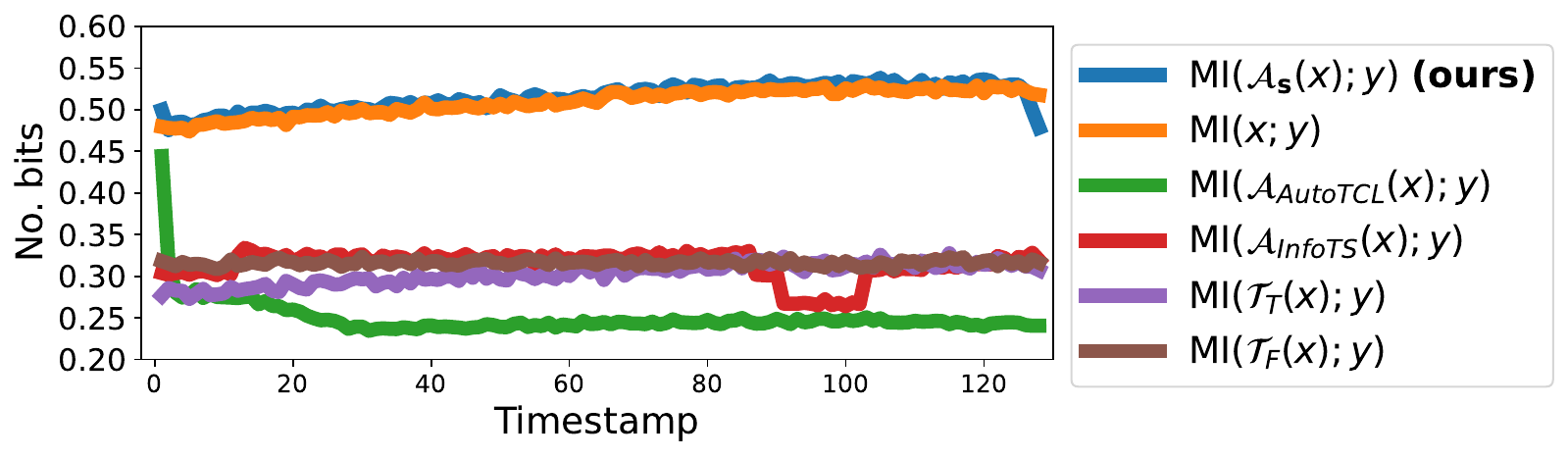}
  \caption{Our method (blue curve) achieves the highest MI between the views generated and the label, enabling better semantic preservation compared with SOTA. The \revise{semantically relevant information} is well preserved to facilitate contrastive representation learning.}
  \label{fig:MI}
  \vspace{-15pt}
\end{figure}

\section{Introduction}

Time series classification has been an essential problem in a wide range of applications, such as activity recognition~\citep{qian2019novel}, speech recognition~\citep{huijben2023som}, and industrial monitoring~\citep{DBLP:journals/pami/EldeleRCWKLG23}. Despite the promising performance achieved by supervised methods~\citep{qian2019novel}, a large number of accurate labels are required to deliver good performance. However, label annotation for time series without human error is costly and time-consuming. This is because time series data are not intuitively recognizable or meaningful for humans, unlike images or language. Given the circumstance, contrastive learning has been attested as a compelling framework for representation learning in the absence of labels~\citep{DBLP:journals/corr/abs-2308-01578,DBLP:conf/kdd/QianTM22}. Specifically, it learns to solve an instance discrimination pretext task~\citep{DBLP:conf/cvpr/WuXYL18} that aims to distinguish different samples (negative pairs) while keeping different views of the same sample (positive pairs) close, wherein different views are usually generated by a set of augmentation functions. 

Despite the prevalence of contrastive learning~\citep{DBLP:conf/cvpr/ChenH21,DBLP:conf/cvpr/HuangZS23}, its efficacy heavily relies on the proper selection of data augmentation~\citep{luo2023time,DBLP:conf/nips/Tian0PKSI20}. 
Existing works in time series contrastive learning often apply carefully hand-picked \revise{time-domain} transformations such as \texttt{jitter-and-scale} and \texttt{permutation-and-jitter}~\citep{DBLP:journals/pami/EldeleRCWKLG23}.  However, these augmentations are mostly adopted from the vision domain and do not take the intrinsic characteristics of time series into consideration. 
Due to the unintuitive nature of time series, it becomes impractical to painlessly figure out semantically compromised augmented samples, unlike in vision. 
\revise{Additionally, predefined stochastic frequency-domain augmentations, such as \texttt{phase-shift} and augmentation in TF-C~\cite{zhang2022self}, introduce semantically irrelevant noise that disrupts the critical information. 
Moreover, frequency-based predefined augmentations, such as high-pass and low-pass filters, require prior knowledge, such as the effective bandwidth of a dataset, to determine the selection of appropriate augmentation functions. As prior knowledge is not always accessible in the contrastive learning paradigm, and the compromised semantics caused by the frequency-domain augmentation, predefined frequency-domain augmentations remain suboptimal for contrastive learning. }
Consequently, when applying predefined augmentation, the type and degree of transformation need to be carefully selected to reduce the loss of semantic information. \revise{However, the reliance on trial-and-error selection of hand-picked augmentations makes the process computationally expensive and impractical. }
What's worse is that there is no single augmentation function that consistently performs well on all diverse datasets~\citep{DBLP:conf/kdd/QianTM22}. 
Given these challenges, recent works have started to explore the generalized principle and design of transformation that yield universally optimal augmentation $\textsf{v}^{\ast}$ for time series contrastive learning. For instance, the latest InfoTS~\citep{luo2023time} and AutoTCL~\citep{DBLP:journals/corr/abs-2402-10434} share a common principle that optimal augmentation should remain semantically consistent with their anchor samples $\text{MI}(\textsf{v}^{\ast}; \textsf{y})=\text{MI}(\textsf{x}; \textsf{y})$, where $\textsf{x}$ and $\textsf{y}$ are the random variable denoting time series sample and \revise{\textit{global} label applied to the entire sequence}, and $\text{MI}(\cdot; \cdot)$ represents the mutual information (MI) quantifying the mutual dependence between two variables. MI is formally defined as $\text{MI}(\textsf{x}; \textsf{y})=\text{H}(\textsf{y})-\text{H}(\textsf{y}|\textsf{x})$, where $\text{H}(\textsf{y})$ and $\text{H}(\textsf{y}|\textsf{x})$ denote the Shannon entropy of $\textsf{y}$ and the entropy of $\textsf{y}$ conditioned on $\textsf{x}$, respectively. However, our empirical observations reveal that these proposed augmentation strategies still fail to preserve semantic integrity. To be more specific, they more or less undermine or disrupt the meaningful patterns associated with \revise{semantically relevant information}, i.e., $\text{MI}(\textsf{v}^{\ast}; \textsf{y})$, of the time series, which will be discussed in detail in the later sections. 
In our analysis, the global semantics, whose amount is quantified as $\text{MI}(\textsf{x}; \textsf{y})$, refer to \textit{the information that spans the entire time series and contributes significantly to distinguishing between different classes}. Therefore, \revise{ensuring the preservation of global semantics is essential for effective view generation in contrastive learning for time series classification tasks.}

For a clearer illustration, we plot 6 different $\text{MI}(\textsf{v}; \textsf{y})$ curves in Figure~\ref{fig:MI}, where $\textsf{v} \in \{\mathcal{A}_\mathbf{s}(\textsf{x}), \textsf{x}, \mathcal{A}_\text{AutoTCL}(\textsf{x}), \mathcal{A}_\text{InfoTS}(\textsf{x}), \mathcal{T}_T(\textsf{x}), \mathcal{T}_F(\textsf{x})\}$, denoting the augmented view generated by our proposed FreRA, identity transformation, AutoTCL, InfoTS and \texttt{jitter-and-scale} and \texttt{amplitude-and-phase-perturbation}, respectively, and $\textsf{y}$ is the label in the downstream classification task. The $x$-axis presents the timestamp of the time series, and the $y$-axis denotes the value of MI. 
We provide more details regarding Figure~\ref{fig:MI} in Appendix~\ref{sec:explain-fig1}. Intuitively, a higher MI curve is preferable because it indicates more global semantics information is preserved. We first observe the proposed FreRA (blue curve) achieves the highest value among all the curves, and it almost overlaps with $\text{MI}(\textsf{x}; \textsf{y})$ (orange curve), indicating FreRA preserves all the \revise{semantically relevant} information in the generated views and there is no major loss of critical information. 
We then observe that the other four curves are consistently lower than the first two curves, indicating \revise{the semantically irrelevant information prevails} and the global semantics are undermined in the latter three transformations, which agrees with our earlier analysis. Previous work~\citep{xu2024self} figures out that undermined semantics in the views cause degraded representation and harm the performance of downstream tasks, which is undesirable.

Despite strong empirical performance on certain datasets, existing augmentations \revise{introduce semantically irrelevant information which undermines semantically relevant context}. This limitation highlights their inherent weaknesses. Due to the \textit{inter-correlation among timestamps}, time-domain manipulations \textit{fail to keep the critical global information intact while introducing variation}.
\revise{Existing frequency-based augmentations do not fully leverage the advantages of the frequency domain. To address this gap, we reanalyze the frequency domain} and identify 3 advantages of it over the time domain: 
1) \textit{global}: each frequency component encapsulates a global feature that spans all timestamps and is more meaningful in revealing the global semantics for classification tasks; 2) \textit{independent}: the orthogonal Fourier basis ensures the independence among frequency components, making it unlikely to contain both critical and noisy information at the same time, which allows clear separation and independent manipulations on different components; and 3) \textit{compact}: given the first two properties, there is a compactly distributed set of frequency components that can well preserve the semantic integrity, \revise{which reduces the risk of information loss while maintaining semantically relevant context}. The three advantages of the frequency domain and how they facilitate the design of FreRA will be elaborated in detail in the latter sections. 

To fully tap into the great potential of the frequency domain, we propose a novel \textbf{Fre}quency-\textbf{R}efined \textbf{A}ugmentation (\textbf{FreRA}) for contrastive learning on time series classification. The central idea of FreRA is to adaptively refine frequency components. Specifically, we learn a lightweight trainable parameter vector to capture the inherent semantic distribution in the frequency domain. \revise{Semantic-aware} identity modification and \revise{semantic-agnostic} self-adaptive modification are then proposed to the well-separated critical and unimportant frequency components respectively, to preserve global semantics and infuse variance. This single-parameter vector adeptly guides the refinement in both the separation and modifications. FreRA is a generalized transformation that automatically adapts to training data, alleviating manual efforts in adjusting augmentations. 
It also ensures that the added variation does not compromise the global semantics by refining the frequency domain rather than the time domain.
FreRA can be easily adapted to a wide range of contrastive learning models in a plug-and-play manner.
In summary, our main contributions are:
\begin{itemize} [leftmargin=*]
\item We identify three advantages of the frequency domain and introduce the novel frequency perspective to automatic view generation for time series contrastive learning for the first time.
\item Building upon these advantages, we design a lightweight and unified automatic augmentation FreRA for contrastive representation learning on classification tasks, which can be applied in a plug-and-play manner and jointly optimized with the contrastive learning model.
\item Extensive experiments on 135 benchmark datasets demonstrate the competitive performance of FreRA in contrastive learning and improved generalization in transfer learning scenarios on both time series classification and anomaly detection tasks.
\end{itemize}

\section{Related Work}
\noindent\textbf{Time series contrastive learning.}
Considering the challenges of data annotation for time series, contrastive learning achieves great success in time series applications~\citep{yue2022ts2vec,DBLP:conf/iclr/TonekaboniEG21,DBLP:journals/pami/EldeleRCWKLG23,meng2023mhccl}. TS2Vec~\citep{yue2022ts2vec} performs hierarchical contrastive learning to learn timestamp-wise representations. TNC~\citep{DBLP:conf/iclr/TonekaboniEG21} learns temporal representations where neighboring and non-neighboring signals are distinguishable. TS-TCC~\citep{DBLP:journals/pami/EldeleRCWKLG23} proposes a novel cross-view prediction task. MHCCL~\citep{meng2023mhccl} utilizes hierarchical clustering for temporal contrastive representation learning.
Although previous works introduce various architectures and objectives, the essence of contrastive learning lies in the attraction of positive pairs and the repulsion of negative pairs~\citep{he2020momentum}, making view generation a crucial component. 

\noindent\textbf{Frequency domain of time series.}
The frequency domain mostly serves as a substitute or supplementary modality in multiple time-series tasks, e.g., representation learning~\citep{yang2022unsupervised}, domain generalization~\citep{zhang2022self}, and time series forecasting~\citep{DBLP:conf/nips/ZhouMWW0YY022,DBLP:conf/icml/ZhouMWW0022, DBLP:conf/nips/YiZFWWHALCN23}. Those works empirically discover and exploit the frequency domain as an informative element: BTSF and TF-C~\citep{yang2022unsupervised,zhang2022self} encourage time-frequency consistency in representation learning to enhance generalization; \citet{DBLP:conf/nips/ZhouMWW0YY022} claim that utilizing low-frequency Fourier components for time series forecasting could undermine noise; \citet{DBLP:conf/icml/ZhouMWW0022} prove that a subset of randomly selected Fourier components preserves most of the information in the time series. \citet{DBLP:conf/nips/YiZFWWHALCN23} find that the frequency domain possessed a global view and compact energy in MLP-based time series forecasting. Those works provide heterogeneous understandings of identifying essential information in the frequency domain, either from domain knowledge or heuristics. 
In contrast, our motivation inspires a unified approach that manipulates frequency-domain information to facilitate contrastive learning. 
\revise{While predefined frequency-domain augmentations have been applied in time series contrastive learning~\citep{DBLP:conf/kdd/QianTM22, zhang2022self}, the full potential of the frequency domain remains underutilized, leaving room for a more strategic augmentation approach. }

\begin{figure*}[!t]
    \centering
    \includegraphics[width=0.85\linewidth]{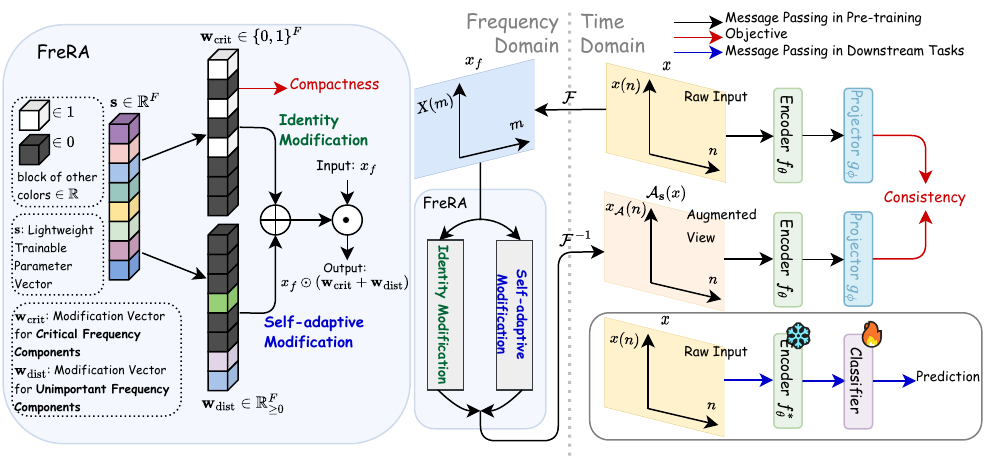}
    \caption{An overview of the proposed FreRA. The left-hand side presents the detailed design of FreRA: \revise{semantic-aware} identity modification on critical components and \revise{semantic-agnostic} self-adaptive modification on unimportant components are conducted in the frequency domain to maintain contextual information and infuse variance respectively. The matching colors between $\mathbf{s}$ and $\mathbf{w}_\text{dist}$ on unimportant components intend to illustrate the adaptive distortion. The independent manipulations in FreRA ensure the added variance does not impact the critical \revise{semantically relevant} information. \revise{$X(m)$ and $x(n)$ represent the frequency domain and the time domain of time series respectively, where $m$ denotes the index of frequency component and $n$ denotes the timestamp index.} As a plug-and-play component, FreRA can be jointly trained with any contrastive learning framework, as illustrated on the right-hand side. The contrastive learning model is pre-trained in the time domain. FreRA encourages the compactness of critical frequency components and the consistency of positive pairs' representations. In evaluation, a classifier is trained on top of the frozen pre-trained encoder to get predictions for downstream tasks. }
    \label{fig:FreRA-overview}
    \vspace{-10pt}
\end{figure*}

\noindent\textbf{Augmentations for contrastive learning.}
% \paragraph{Augmentations for Contrastive Learning. }
As a crucial component for contrastive learning, augmentation functions are either carefully designed or selected from grid search~\citep{DBLP:conf/kdd/QianTM22,DBLP:journals/pami/EldeleRCWKLG23}. The former requires domain knowledge, while the latter is computationally inefficient. There is no single existing augmentation function enjoying universal optimal performance~\citep{DBLP:conf/kdd/QianTM22}. The selection is task-dependent~\citep{DBLP:conf/nips/Tian0PKSI20} and subject to data modality~\citep{jaiswal2020survey}. Some works try to automate the selection from a predefined set of transformations or adapt a well-defined transformation to serve contrastive learning: 
% \citet{miao2023learning} attempts to sample instance-specific augmentation by capturing local invariance on the image, and 
InfoTS~\citep{luo2023time} trains a data-driven probabilistic augmentation selector that intends to encourage high fidelity and variety to select optimal augmentation. 
% The formal one is specially designed for the image domain thus the assumptions made do not apply to time series. 
\citet{demirel2024finding} introduce tailored mixup for non-stationary quasi-periodic time series.
% Note that all the abovementioned automatic augmentation approaches require a predefined set of transformation functions, which are hand-picked and require careful calibration. 
Another line of work eliminates the use of hand-designed augmentation: InfoMin~\citep{DBLP:conf/nips/Tian0PKSI20} generates contrastive views with a flow-based model, guided by the adversarial InfoMin objective. 
% However, the objective is sometimes overfitted and the training is unstable. 
AutoTCL~\citep{DBLP:journals/corr/abs-2402-10434} factorizes the time series instance to informative and noisy parts by timestamps. 
% Similar to InfoTS and InfoMin, it also requires alternative optimization. 
Self-adaptive augmentation in the frequency domain is less explored in contrastive learning, and we fill this research gap in this work. 
% An automatic augmentation strategy that gets rid of any hand-designed augmentation functions, can be optimized efficiently and achieves universal good representation learning performance on contrastive learning has yet to be found. To this end, we provide a new angle from the frequency domain where the conditions can be satisfied, which is less explored in earlier works. 

\section{Methodology}

% \subsection{Problem Statement}

% \textbf{Notations}
% time series and frequency domain
% Given a mini-batch of $B$ instances $X \in \mathbb{R}^{B \times L \times D}$,  let $x \in X$ be an unlabeled sample in the mini-batch denoting a time series that lasts for $L$ timestamps and has $D$ channels, i.e., $x \in \mathbb{R}^{L \times D}$. 
% Let $x \in \mathbb{R}^{L \times D}$ be an unlabeled sample denoting a time series that lasts for $L$ timestamps and has $D$ channels.

% Let $x \in X$ denote an unlabeled time series instance that lasts for $L$ timestamps and has $D$ channels. $\mathcal{F}(\cdot)$ and $\mathcal{F}^{-1}(\cdot)$ represent the Fourier transform and its inverse. $x^{f} = \mathcal{F}(x) \in \mathbb{C}^{F \times D}$ denotes the Fourier transform of $x$, where $\mathbb{C}$ stands for the complex space and $F=\lfloor{L/2}\rfloor+1$ is the number of Fourier components. 

\subsection{Problem Statement}
% Let $x=\{x_t\}_{t=1}^{T} \in \mathbb{R}^{T \times D}$ denote an unlabeled mulitvariate time series instance that lasts for $T$ timestamps and has $D$ channels where $x_t \in \mathbb{R}^{D}$. $\mathcal{F}(\cdot)$ and $\mathcal{F}^{-1}(\cdot)$ represent the Fourier transform and its inverse. We denote $\{x_f\}_{f=1}^{F}=\mathcal{F}(x) \in \mathbb{C}^{F \times D}$ as the Fourier transform of $x$, where $\mathbb{C}$ stands for the complex space and $F=\lfloor{L/2}\rfloor+1$ is the number of Fourier components. 
Let $x=[x^1, x^2, ..., x^L]^T \in \mathbb{R}^{L \times D}$ denote an unlabeled time series instance that lasts for $L$ timestamps and has $D$ channels where the signal at $i$-th timestamp $x^i \in \mathbb{R}^{D}, \forall i \in [1,..., L]$. We do not make assumptions about the dimension or length of the time series. Our problem definition is valid for both univariate and multivariate time series datasets of varying scales. $\mathcal{F}(\cdot)$ and $\mathcal{F}^{-1}(\cdot)$ represent the Fourier transform and its inverse, respectively. We denote $x_f = [x_f^1, x_f^2, ..., x_f^F]^T=\mathcal{F}(x) \in \mathbb{C}^{F \times D}$ as the Fourier transform of $x$, i.e., $x = \mathcal{F}^{-1}(\mathcal{F}(x))$, where $\mathbb{C}$ stands for the complex space and $F=\lfloor{L/2}\rfloor+1$ is the number of frequency components. $x_f^1$ and $x_f^F$ embed the characteristics of the lowest and highest Fourier frequency basis functions, respectively.
% Since the Fourier transform is an invertible operation, we have $x=\mathcal{F}^{-1}(\{x_f\}_{f=1}^{F})$.

In the general contrastive learning framework, an encoder $f_\theta$ is trained to map input samples to a latent space where the downstream task is performed. Taking SimCLR~\citep{DBLP:conf/icml/ChenK0H20} as our contrastive learning framework, it appends a projector $g_\phi$ to the encoder. $\theta$ and $\phi$ denote the sets of trainable parameters in the encoder and projector respectively. 
In the mini-batch $X \in \mathbb{R}^{B \times L \times D}$ containing $B$ instances, each anchor $x \in X$ associates with its augmented view $\mathcal{A}(x)$ as a positive pair, and with the other $(B-1)$ samples to form negative pairs.
We consider the batch-wise contrastive loss as: $\mathcal{L}_{\text{CL}} = \mathcal{L}(X; \mathcal{A}(\cdot), f_\theta, g_\phi)$, which will be elaborated later. 

It is a common belief in existing works~\citep{DBLP:conf/nips/Tian0PKSI20,luo2023time,DBLP:journals/corr/abs-2402-10434} that the optimal view generator for contrastive learning is defined as follows.
\begin{definition}[\textbf{Optimal View Generator}]
Given the random variable $\textsf{x}$ denoting the input instances, its optimal view generator $A^\ast(\cdot)$ should satisfy $\mathcal{A}^\ast(\textsf{x}) = \underset{\mathcal{A}}{\arg\min} \text{ MI}(\mathcal{A}(\textsf{x}); \textsf{x})$, subject to $\text{MI}(\mathcal{A}^\ast(\textsf{x}); \textsf{y}) = \text{MI}(\textsf{x}; \textsf{y})$.
\label{definition:optimal-view}
\end{definition}
Based on the definition, an optimal view generator should preserve the \textit{minimal but sufficient} information with respect to the global semantics of its input.
Existing works on time series contrastive learning mainly select an empirically optimal augmentation function $\mathcal{T}^{\ast}(\cdot)$ from a set of predefined transformations $\{\mathcal{T}_1(\cdot), \mathcal{T}_2(\cdot), ..., \mathcal{T}_m(\cdot)\}$, such as \{\texttt{scaling}, \texttt{jittering}, \texttt{rotation}\}, i.e., 
% $\mathcal{T}^{\ast}(\cdot), \theta^{\ast}, \phi^{\ast} = \underset{\mathcal{T}_i(\cdot) \in \{\mathcal{T}_1(\cdot), \mathcal{T}_2(\cdot), ..., \mathcal{T}_m(\cdot)\}, \theta, \phi}{\arg\min}\mathcal{L}(X; \mathcal{T}_i(\cdot), f_\theta, g_\phi).$
\begin{equation}
    \mathcal{T}^{\ast}(\cdot), \theta^{\ast}, \phi^{\ast} = \underset{\mathcal{T}_i(\cdot) \in \{\mathcal{T}_1(\cdot), \mathcal{T}_2(\cdot), ..., \mathcal{T}_m(\cdot)\}, \theta, \phi}{\arg\min}\mathcal{L}(X; \mathcal{T}_i(\cdot), f_\theta, g_\phi).
\end{equation}
Selected from the painstaking trials and errors, $\mathcal{T}^{\ast}(\cdot)$ still suffers from loss of semantic information. 
Other works utilize a trainable network to model the transformation function, denoted as $\mathcal{T}(\cdot; \gamma)$, where $\gamma$ is the parameters of the transformation network. They optimize the entire framework as follows: 
% $\gamma^{\ast}, \theta^{\ast}, \phi^{\ast} = \underset{\gamma}{\arg\max} \text{ } \underset{\theta, \phi}{\arg\min} \text{ } \mathcal{L}(X; \mathcal{T}(\cdot; \gamma), f_\theta, g_\phi) + \mathcal{L}_\text{auxiliary}(\gamma),$
\begin{equation}
    \gamma^{\ast}, \theta^{\ast}, \phi^{\ast} = \underset{\gamma}{\arg\max} \text{ } \underset{\theta, \phi}{\arg\min} \mathcal{L}(X; \mathcal{T}(\cdot; \gamma), f_\theta, g_\phi) + \mathcal{L}_\text{auxiliary}(\gamma),
\end{equation}
where $\mathcal{T}(\cdot; \gamma^\ast)$ denotes the learned transformation function and $\mathcal{L}_\text{auxiliary}(\gamma)$ is the extra regularization on the transformation network. The optimization for the min-max objective is done through an alternative update of the transformation network and the contrastive learning model.

Aware of the selection cost, compromised semantics, and the complex alternative optimization in previous approaches, we aim to develop a semantic-preserving automatic augmentation $\mathcal{A}(\cdot)$ that can be jointly optimized with the contrastive learning model, with objective formulated as follows:
\begin{equation}
\begin{split}
    & \underset{\mathcal{A}(\cdot), \theta, \phi}{\arg\min} \mathcal{L}(X; \mathcal{A}(\cdot), \theta, \phi) + \mathcal{L}_\text{auxiliary}(\mathcal{A}(\cdot)) \\
    & \text{subject to } \text{MI}(\mathcal{A}(\textsf{x}); \textsf{y}) = \text{MI}(\textsf{x}; \textsf{y}).
\end{split}
\label{Equation:general-obj}
\end{equation}
In the following section, we will present the design of $\mathcal{A}(\cdot)$ and prove the fulfillment of Definition~\ref{definition:optimal-view} as well as the criterion in the objective. 

% , with objective formulated as follows:
% \begin{equation}
% \begin{split}
%     & \underset{\mathcal{A}(\cdot), \theta, \phi}{\arg\min} \mathcal{L}(X; \mathcal{A}(\cdot), \theta, \phi) + \mathcal{L}_\text{auxiliary}(\mathcal{A}(\cdot)) \\
%     & \text{subject to } \text{MI}(\mathcal{A}(\textsf{x}); \textsf{y}) = \text{MI}(\textsf{x}; \textsf{y}).
% \end{split}
% \label{Equation:general-obj}
% \end{equation}

\subsection{Overview of FreRA}
It is a common belief that a good view in contrastive learning should contain both semantic-preserving information and a considerable amount of variance~\citep{DBLP:journals/corr/abs-2402-10434,luo2023time}. The former ensures strong performance on downstream tasks, while the latter encourages the encoder to learn generalizable representations.
To achieve such a good view, we leverage the global, independent, and compact properties of the frequency domain to design the frequency-refined augmentation, FreRA as follows:
\begin{equation}
    \mathcal{A}_\mathbf{s}(x) = \mathcal{F}^{-1}(\underbrace{\overbrace{\mathbf{w}_{\text{crit}} \odot x_f}^\text{global and compact} + \mathbf{w}_\text{{dist}} \odot x_f}_\text{independent}) \in \mathbb{R}^{L \times D},
\end{equation}
% \begin{equation}
%     \mathcal{A}(x) = \mathcal{F}^{-1}((\underbrace{\overbrace{\mathbf{w}_{\text{crit}}}^\text{global and compact} + \mathbf{w}_\text{{dist}}}_\text{independent}) \odot x_f) \in \mathbb{R}^{L \times D},
% \end{equation}
\sloppy where $\odot$ denotes elementwise multiplication and $\mathbf{s}$ is the lightweight trainable parameter of FreRA. Specifically, $\mathbf{w}_{\text{crit}} = [w_{\text{crit}}^1, w_{\text{crit}}^2, ..., w_{\text{crit}}^F] \in \{0, 1\}^F$ applies \revise{semantic-aware} identity modification on those identified critical frequency components to preserve the essential \revise{semantically relevant} information, while $\mathbf{w}_\text{{dist}} \in \mathbb{R}_{\ge 0}^F$ applies \revise{semantic-agnostic} self-adaptive modification to the unimportant frequency components to introduce diverse distortion. The two modifications are applied independently to keep the critical global information intact while introducing variance. There may exist certain component $x_f^i$ whose $w_{\text{crit}}^i$ and $w_{\text{dist}}^i$ are both 0. The refinement, including the component separation and modifications, is guided by a single vector $\mathbf{s}$. Figure~\ref{fig:FreRA-overview} depicts an overview of the proposed FreRA. Further details are elaborated in Section~\ref{sec:frera_ddetail}.

\subsubsection{Why FreRA Makes Good Views?}
In this section, we elaborate the three advantages of the frequency domain over the time domain and elaborate on them in detail. Based on them, we explain why the frequency-domain refinement produces good views that benefit contrastive representation learning for downstream tasks. To facilitate our analysis, we introduce a new set of notation for time-domain data $x(n)$, augmented view $x_\mathcal{A}(n)$, and frequency-domain data $X(m)$ as follows:
\begin{equation}
\resizebox{0.46\textwidth}{!}{$
\begin{aligned}
    x(n) & = x^{n+1}, \quad x_\mathcal{A}(n) = \mathcal{A}_\mathbf{s}(x^{n+1}) \quad \text{ for } n \in \{0, 1, ..., L-1\},  \\
    X(m) & = \begin{cases}
    x_f^{m+1}, & \text{if $m+1 \le F$}\\
    \overline{X(L - m)} =  \overline{x_f^{L-m+1}}, & \text{otherwise},
  \end{cases} \text{ for } m \in \{0, 1, ..., L-1\}\\
  % x_\mathcal{A}(n) & = \mathcal{A}_\mathbf{s}(x^{n+1}) \text{ for } n \in \{0, 1, ..., L-1\},
\end{aligned}
$}
\label{Equation:annotation-conv}
\end{equation}
where $\overline{x_f^{L-m+1}}$ is the conjugate of $x_f^{L-m+1}$, $x = [x(0), x(1), ..., x(L-1)]^T$ and $x_f = [X(0), X(1), ..., X(F-1)]^T$. We present the derivation for the second condition of $X(m)$ in Appendix~\ref{sec:conj-sym}. 

\noindent\textbf{Global.}
% \paragraph{Global.}
% \subsubsection{Global.}
The Fourier component is derived by Discrete Fourier Transform (DFT)~\citep{sundararajan2001discrete}: $X(m)=\sum_{n=0}^{L-1} x(n) e^{-\frac{2 \pi i}{L} m n}$, where each frequency component $X(m)$ encodes all the timestamps.
According to the Dual convolution theorem~\citep{sundararajan2001discrete}, element-wise multiplication in the frequency domain is equivalent to circular convolution in the time domain. Then we have $\mathcal{F}(\tilde{\mathbf{w}} \ast x) = \frac{1}{F} \mathcal{F}(\tilde{\mathbf{w}}) \odot \mathcal{F}(x)$,
% \begin{equation}
%     \mathcal{F}(\tilde{\mathbf{w}} \ast x) = \frac{1}{F} \mathcal{F}(\tilde{\mathbf{w}}) \odot \mathcal{F}(x),
% \end{equation}
where $\ast$ denotes the circular convolution operator. Let $\mathcal{F}(\tilde{\mathbf{w}}) = \mathbf{w}_{\text{crit}}$, we can conclude that the frequency modification is equivalent to time-domain convolution with kernel $\tilde{\mathbf{w}} = \mathcal{F}^{-1}(\frac{1}{F}\mathbf{w}_{\text{crit}}) \in \mathbb{C}^L$, which has global receptive field on $x$. This global perspective is crucial to the time series classification tasks, as it \textit{preserves global semantics across the entire time series and ensures that all timestamps are altered with distortion applied only to unimportant components}.
% Note that the global property of the frequency component is not only helpful in preserving global semantics but also ensures that all the timestamps are altered with distortion applied to unimportant components.

\noindent\textbf{Independent.}
% \paragraph{Independent.}
% \subsubsection{Independent.}
The inverse DFT, $x(n)=\frac{1}{L}\sum_{m=0}^{L-1} X(m) e^{\frac{2 \pi i}{L} m n}$, offers an alternative perspective of interpreting the frequency components: they are the coefficients of the orthogonal decomposition of the time domain. The decomposition basis for $X(m)$ is $\mathbf{u}_m = [e^{\frac{2 \pi i}{L} m n} | n=0, 1, ..., L-1]^T \in \mathbb{C}^L$. We have $\langle \mathbf{u}_m, \mathbf{u}_q \rangle = 0$ if $m \neq q$, where $\langle \mathbf{u}, \mathbf{v} \rangle = \mathbf{u}^T \overline{\mathbf{v}} \in \mathbb{C}$ is the Hermitian inner product. The proof is presented in Appendix~\ref{sec:orthogonal}. The zero-valued Hermitian inner product confirms the orthogonal nature of the decomposition basis. Each coefficient $X(m)$ measures the contribution of its corresponding basis function independently.
Similarly, when FreRA modifies frequency components, 
% i.e., $x_\mathcal{A}(n) = \sum_{m=0}^{L-1} (\mathbf{w}_{\text{crit}} + \mathbf{w}_\text{{dist}}) \odot X(m) e^{\frac{2 \pi i}{L} m n}$, 
each modified components independently contribute to the augmented views without being affected by the others. The independence \textit{makes it easy to isolate critical and unimportant information} by updating $\mathbf{w}_{\text{crit}}$ and $\mathbf{w}_\text{{dist}}$ and \textit{prevent added variance from degrading critical information}.
% As the coefficient measuring the contribution of each orthogonal decomposition basis, $X(m)$ is independent of the other frequency components.

% , $x^n=\sum_{m=1}^{L} x_f^m e^{\frac{2 \pi i}{L} m n}$. One important property is that any two Fourier bases are orthogonal, i.e., $\int_0^{2 \pi} e^{\frac{2 \pi i}{L} p n} e^{\frac{2 \pi i}{L} q n} dn = 0$ if $p \neq q$. This makes it easy to isolate critical information from the unimportant noise in the frequency domain because orthogonal bases are uncorrelated. 

\noindent\textbf{Compact.}
% \paragraph{Compact.}
Parseval's theorem~\citep{parseval1806memoire} states that the total energy of the signal in the time domain is equal to the average energy in the frequency domain, i.e.,  $\sum_{n=0}^{L-1} \lvert x(n)\rvert^2=\frac{1}{L} \sum_{m=0}^{L-1} \lvert X(m) \rvert ^2.$
This implies that if most energy is concentrated in a small number of frequency components, the information of the signal is compactly distributed in the frequency domain.
% It suggests that if a small amount of frequency components contain most of the energy, this compact set can recover most of the information of the original signal. 
Figure~\ref{fig:energy} in the Appendix validates this interpretation by showing that most energy is concentrated on the first ten frequency components for the UCIHAR dataset and the same principle holds for other datasets. 
This aligns with our common sense that many natural or man-made processes recorded as time series encode information in low-frequency components. However, for classification tasks, the semantically relevant bandwidth is normally unknown and the importance of these components varies, making it hard to automatically rank their contributions and identify critical ones. Moreover, the exception happens in certain applications, such as audio processing~\citep{virtanen2015compositional}, where both low- and high-frequency components matter. 
As critical components that encapsulate the semantically relevant information of the signal are likely a subset of the compactly distributed informative components, their distribution should also remain compact. 
% It is also important to distinguish between the \textit{compact} set and the \textit{critical} set. The former contains the majority of the information while the latter \textit{encapsulates the semantic meaning of the signal}. 
This leads us to \textit{enforce compactness in identifying the critical components in the frequency domain}. Notably, the range of the energy in Figure~\ref{fig:energy} highlights the shared distribution of the informative components in the frequency domain. It advocates that a single vector $\mathbf{s}$ is sufficient to work across all the samples in the dataset.
% \subsubsection{Compact.}

Lastly, we demonstrate that FreRA preserves global semantics of the time series, i.e., $\text{MI}(\mathcal{A}_\mathbf{s}(\textsf{x}) ; \textsf{y}) = \text{MI}(\textsf{x}; \textsf{y})$ (Proposition~\ref{proposition:cri_sem} in the Appendix with proof) under the reliable assumption that noisy frequency components are independent to the label.
% \textbf{Proposition 3.} \textit{With the reliable assumption that the noisy frequency components are independent to the label, FreRA is a semantic preserving transformation, i.e., $\text{MI}(\mathcal{A}_\mathbf{s}(\textsf{x}) ; \textsf{y}) = \text{MI}(\textsf{x}; \textsf{y})$.}
% The proof is presented in the Appendix~\ref{proof:semantic-preserving}.
This proposition agrees with our observation in Figure~\ref{fig:MI} where the blue and orange curves nearly overlap. It also shows that FreRA satisfies the semantic-preserving constraint in Definition~\ref{definition:optimal-view}, leaving only the minimization objective for optimization.

\subsection{Time Series Contrastive Learning with FreRA}
\label{sec:frera_ddetail}
In this section, we first elaborate on the detailed design of FreRA and propose the objective that allows the joint training of FreRA and the contrastive learning framework.

\noindent\textbf{Discern the importance of frequency components.}
% \paragraph{Discern the Importance of Frequency Components}
Both $\mathbf{w}_{\text{crit}}$ and $\mathbf{w}_\text{{dist}}$ are parameterized by a lightweight trainable vector $\mathbf{s}=[s_1, s_2, ..., s_F]^{T} \in \mathbb{R}^{F}$, where $s_i$ scores the importance of the $i$-th frequency component $x_f^i$ for the global semantics. A higher $s_i$ indicates the contextual importance of $x_f^i$. On the other hand, $s_i$ with a negative value suggests $x_f^i$ is the noise component. 

\noindent\textbf{\revise{Semantic-aware} identity modification on critical frequency components.}
% \paragraph{Identity Modification on Critical Frequency Components.}
A simple way to derive a binary vector like $\mathbf{w}_{\text{crit}}$ is to sample from a Bernoulli distribution controlled by the parameter $\mathbf{p}=[p_1, p_2, ..., p_F]^{T} \in \mathbb{R}^{F}$, i.e., $w_{\text{crit}}^i \sim \operatorname{Bernoulli}(p_i)$ for $i \in [1, 2, ..., F]$, where $p_i$ denotes the probability that the $i$-th frequency component is semantically critical. Meanwhile, the Bernoulli distribution is not differentiable w.r.t. $p_i$. Instead, we apply the Gumbel-Softmax reparameterization~\citep{DBLP:conf/iclr/JangGP17}, i.e.,  $w_{\text{crit}}^i = \text{Gumbel-Softmax}(p_i)$. The importance score vector $\mathbf{s}$ makes it possible because its values can be used to reflect the probability, i.e., $p_i = \sigma(s_i)$, where $\sigma(\cdot)$ is the sigmoid function. The reparameterization is formulated as follows:
\begin{equation}
    w_{\text{crit}}^i = \sigma((\log \epsilon-\log (1-\epsilon)+\log \frac{\sigma(s_i)}{1-\sigma(s_i)}) / \tau_w),
\label{fig:reparameterization}
\end{equation}
where $\epsilon \sim \operatorname{Uniform}(0,1)$ and $\tau_w$ is the temperature controlling the discretization. As $\tau_w \rightarrow 0$, $w_{\text{crit}}^i$ approximates a Bernoulli distribution: $P(w_{\text{crit}}^i \rightarrow 0) = 1-p_i  \text{ if } \epsilon > p_i$, and $P(w_{\text{crit}}^i \rightarrow 1) = p_i  \text{ if } \epsilon < p_i$. 
% \begin{equation}
% \begin{gathered}
%     P(w_{\text{crit}}^i \rightarrow 0) = 1-p_i  \text{ if } \epsilon > p_i \quad
%     P(w_{\text{crit}}^i \rightarrow 1) = p_i  \text{ if } \epsilon < p_i.
% \end{gathered}
% \end{equation}
In this way, distinct importance score $s_i$ is assigned to $x_f^i$ to capture varying levels of contextual relevance within each frequency component. 

\noindent\textbf{\revise{Semantic-agnostic} self-adaptive modification on unimportant frequency components.}
% \paragraph{Self-adaptive Modification on Unimportant Frequency Components.}
Besides preserving contextually relevant information, a good view also requires variance to be infused. Instead of adding random noise, we deliberately modify the semantically irrelevant noisy components identified by  $\mathbf{s}$ to avoid affecting critical information. 
As the score $s_i$ indicates the importance of the $i$-th frequency component $x_f^i$ for global semantics, frequency components with smaller values are considered unimportant. A threshold value is required to separate the unimportant components from the rest and handpicking such a value would be inefficient and troublesome due to its dataset-specific nature. A practical approach is to determine the value with statistical information of the vector. In this work, we use the mean value for convenience. We empirically compare the performance using alternative thresholds in Appendix~\ref{sec:alternativeD}.
Let $D = \{i | s_i < \min(0, \frac{1}{F} \sum_{i=1}^{F}{s_i})\}$ denote the set of unimportant components' indices. 
Finding the minimum between the mean value and 0 ensures the threshold is non-positive. This is to prevent components with positive scores from being sampled. 
The distortion vector $\mathbf{w}_\text{dist} = \frac{1}{\delta_s} \mathbbm{1}_{\{i \in D\}} \odot \lvert \mathbf{s} \rvert \in \mathbb{R}_{\ge 0}^F$ modifies the unimportant frequency components to various extent. The scaling factor $\delta_s = \frac{1}{\lvert D \rvert} \sum_{i=1}^{F} \mathbbm{1}_{\{i \in D\}} \lvert s_i \rvert$ controls the degree of distortion such that it is in accordance to each component's insignificance and no dramatic interference will be introduced. Because of the absolute value function, the least important frequency component gets amplified mostly in the distortion step. Lastly, we apply stop-gradient operation, i.e., $\mathbf{w}_\text{dist} = \operatorname{stopgrad}(\mathbf{w}_\text{dist})$ because back-propagation is not desired for the distortion. Data-driven thresholding and scaling define the self-adaptive nature of modification on semantically irrelevant frequency components. By modifying these components, variance is infused into all timestamps.

\noindent\textbf{Overall objective.}
% \paragraph{Overall Objective.}
The Gumbel-Softmax reparameterization makes $\mathbf{w}_{\text{crit}}$ differentiable, which allows the joint training of automatic augmentation and the contrastive learning framework. 
Specifically, the contrastive model is optimized by pulling positive pairs together and pushing negative pairs apart through the InfoNCE loss ~\citep{DBLP:journals/corr/abs-1807-03748}, given by:
\begin{equation}
    \mathcal{L}_{\text{CL}} = -\frac{1}{B} \sum_{x \in X} \log \frac{\exp(\text{sim}(h_x, \hat{h}_x)/\tau)}{\sum_{x\prime \in X}\exp(\text{sim}(h_x, \hat{h}_{x\prime})/\tau)},
\end{equation}
where $h_x = g_\phi(f_\theta(x))$, $\hat{h}_x = g_\phi(f_\theta(\mathcal{A}_\mathbf{s}(x)))$,  sim$(\cdot,\cdot)$ denotes the similarity measurement implemented as the cosine similarity and $\tau$ is the temperature coefficient.
Minimizing the InfoNCE loss is equivalent to maximizing the lower bound $\text{MI}_\text{CL}(\textsf{x}, \mathcal{A}_\mathbf{s}(\textsf{x}))$ of the mutual information $\text{MI}(\textsf{x}, \mathcal{A}_\mathbf{s}(\textsf{x}))$~\citep{DBLP:journals/corr/abs-1807-03748}, i.e., $\text{MI}(\textsf{x}, \mathcal{A}_\mathbf{s}(\textsf{x})) \le \log(B) - \mathcal{L}_\text{CL} = \text{MI}_\text{CL}(\textsf{x}, \mathcal{A}_\mathbf{s}(\textsf{x}))$, where $B$ denotes the batch size. For $\mathcal{A}_\mathbf{s}(\cdot)$, directly applying InfoNCE results in a trivial solution of $\mathbf{s}$ that causes $\mathbf{w}_{\text{crit}}$ to become an all-one vector $\mathbf{1} \in \{1\}^F$, leaving the importance of frequency components ambiguous. This is because $x = \mathcal{F}^{-1}(\mathbf{1} \odot x_f)$. 
On the other hand, the critical components should keep and only keep the \revise{semantically relevant} information, as the name suggests, i.e., $\text{MI}(\textsf{x}_{\text{crit}}; \textsf{x}) = \text{MI}(\textsf{x}; \textsf{y})$, where $\textsf{x}_{\text{crit}} = \mathcal{F}^{-1}(\mathbf{w}_{\text{crit}} \odot \textsf{x}_f)$. Knowing that DFT is a reversible operation, we prove $\text{MI}(\textsf{x}_{\text{crit}}; \textsf{x})=\text{MI}(\mathbf{w}_{\text{crit}} \odot \textsf{x}_f; \textsf{x}_f)$ and $\text{MI}(\textsf{x}; \textsf{y})=\text{MI}(\textsf{x}_f; \textsf{y})$ in the Appendix~\ref{proof:cons-MI}. The orthogonal property of the Fourier basis reminds us that the frequency components are uncorrelated. In other words, $\text{MI}(\mathbf{w}_{\text{crit}} \odot \textsf{x}_f; \textsf{x}_f))$ keeps increasing as a higher proportion of frequency components are identified as critical ones, as illustrated in Figure~\ref{fig:MI-crit} in the Appendix, with proof provided in the Appendix~\ref{proof:mono-increase}. The trivial solution falls on the right end of the red segment and the optimal proportion of critical components is pointed by the green arrow.
To avoid the trivial solution and to achieve a good view, we regularize the proportion of critical components in complement to the InfoNCE loss. Specifically, we employ the L1-norm on $\mathbf{w}_{\text{crit}}$ as follows:
\begin{equation}
    \mathcal{L}_{\text{reg}} = \frac{1}{F}\sum_{f = 1}^{F}{\lvert w_{\text{crit}}^f \rvert}.
\end{equation}
The regularization eliminates redundancy from identifying too many critical frequency components, leading to compact selection and robust representation learning. 
The overall optimization problem is given by:
\begin{equation}
\begin{gathered}
    \mathbf{s}^\ast, \theta^\ast, \phi^\ast = 
    \underset{\mathbf{s}, \theta, \phi}{\arg\min} %_{\theta, \phi, \mathbf{w}} 
    (\mathcal{L}_{\text{CL}} + \lambda \cdot \mathcal{L}_{\text{reg}}), \\
\end{gathered}
\label{eq:overall-objective}
\end{equation}
where $\lambda$ is a hyper-parameter to balance the two losses. Note that there exists a unique value of critical component's proportion that makes $\text{MI}(\textsf{x}_{\text{crit}}; \textsf{x}) = \text{MI}(\textsf{x}; \textsf{y})$ happen, as shown in Figure~\ref{fig:MI-crit}. 
% This is because a linear system with a monotonically increasing function and a constant function, whose value lies within the valid range of the increasing function, has a unique solution. 
Meanwhile, as the hyper-parameter regularizes the proportion, $\lambda$ empirically exhibits stable performance over a range of values, as shown in Figure~\ref{fig:param}.

\noindent\textbf{How does the learning objective benefit view generation?}
% \paragraph{How Does the Learning Objective Benefit View Generation?}
\revise{Refining frequency components offers a key advantage in preserving global semantics, as each frequency component inherently encodes global information that spans the entire time series. Unlike time-domain augmentations, which disrupt inter-correlations among timestamps and compromise global semantics, FreRA leverages the global property of the frequency domain and independently modifies critical and unimportant components. This ensures that FreRA preserves global semantics while introducing proper variance. However, refining frequency components is inherently challenging due to the absence of ground truth labels, making it impractical to directly identify the optimal set of semantically relevant frequency components. Instead, FreRA relies on the contrastive learning objective to learn how to refine frequency components in a data-driven manner, effectively distinguishing between critical and unimportant information. }
From information theory, optimizing Eq.~\ref{eq:overall-objective} is equivalent to maximize the lower bound for $\text{MI}(\textsf{x}, \textsf{x}_\text{crit})$ and minimize $\text{MI}(\mathcal{A}_\mathbf{s}(\textsf{x}), \textsf{x})$. The former occurs because optimizing $\mathbf{s}$ over the InfoNCE loss only maximizes the lower bound for $\text{MI}(\textsf{x}, \textsf{x}_\text{crit})$, due to the stop-gradient operation applied to the unimportant frequency component. The latter is achieved by the regularization term and the distortion applied to unimportant components. Combined with the Proposition~\ref{proposition:cri_sem} we have proved earlier, we prove the view generator trained on objective in Eq.~\ref{eq:overall-objective} satisfies the optimality as defined in Definition~\ref{definition:optimal-view}. 
\revise{Empirically, our results demonstrate that our proposed frequency-domain refinement significantly outperforms random sampling-based frequency modifications, validating the effectiveness of our approach in preserving global semantics while introducing meaningful variations.}

\begin{table*}[!t]
\caption{The overall performance on all the datasets (unit: \%). best(T) and best(F) record the highest performances among the selected sets of 11 time-domain augmentations and 5 frequency-domain augmentations. The best performance is highlighted in \textbf{bold}, and the second-best performance is {\ul underlined}. $^\ast$ indicates FreRA significantly outperforms both best(T) and best(F) at the confidence level of 0.05 from paired t-test.}
\label{Table:result-overall}
\centering
% \tiny
\footnotesize
% \small
% \scriptsize
% \resizebox{\textwidth}{!}{
\begin{tabular}{c|c|ccccccccccc}
\toprule
Dataset                       & Metrics   & \begin{tabular}[c]{@{}c@{}}FreRA\\  \textbf{(ours)}\end{tabular}          & best(T) & best(F) & InfoMin$^+$ & InfoTS & AutoTCL & TS2Vec               & TNC   & TS-TCC & TF-C & SoftCLT  \\ \midrule
UCIHAR                        & ACC       & \textbf{0.975}$^\ast$ & 0.959   & 0.960   & {\ul 0.967}   & {\ul 0.967}  & 0.697   & 0.959                & 0.568 & 0.924  & 0.875 & 0.961\\\midrule
MS                            & ACC       & \textbf{0.982}$^\ast$ & 0.956   & 0.970   & {\ul 0.971}   & 0.967  & 0.691   & 0.945 & 0.526 & 0.915  & 0.811 & 0.962\\\midrule
WISDM                         & ACC       & \textbf{0.972}$^\ast$ & 0.942   & 0.950   & {\ul 0.959}   & 0.915  & 0.760   & 0.939 & 0.543 & 0.889  & 0.839 & 0.952\\\midrule
\multirow{2}{*}{\begin{tabular}[c]{@{}c@{}}UEA\\  Archive\end{tabular}} & Avg. ACC  & \textbf{0.754}$^\ast$ & 0.684          & 0.686          & 0.693          & 0.714          &  0.742          & 0.704          & 0.670          & 0.668          & 0.298 & {\ul0.751}\tablefootnote{The result is directly adopted from its original paper. As the results across all the datasets in the UEA and UCR archives are not provided, the ranking is not available.\label{reference footnote}}\\
                              & Avg. RANK & \textbf{2.133} & 5.967          & 5.800          & 5.500          & 3.967          & {\ul 2.600}          & 4.967          & 6.433          & 6.033          & 9.276 & -\\ \midrule
\multirow{2}{*}{\begin{tabular}[c]{@{}c@{}}UCR\\  Archive\end{tabular}}  & Avg. ACC  & \textbf{0.850}$^\ast$ & 0.723   & 0.744   & 0.718   & {\ul 0.849}    & 0.598   & 0.845  & 0.776 & 0.780  & 0.542 & \textbf{0.850}\textsuperscript{\getrefnumber{reference footnote}}\\
                              & Avg. RANK & {\ul 1.940}    & 6.320   & 5.750   & 6.470   & \textbf{1.930} & 8.420   & 2.670  & 4.810 & 4.670  & 8.330 & -\\                           \bottomrule
                             
\end{tabular}
% }
% \vspace{-5pt}
\end{table*}

\noindent\textbf{Distinction to existing automatic augmentation for time series contrastive learning.}
% \paragraph{Distinction to Existing Automatic Augmentation for Time Series Contrastive Learning}
% At first glance, our method may seem similar to InfoTS~\citep{luo2023time} and AutoTCL~\citep{DBLP:journals/corr/abs-2402-10434}, but FreRA fundamentally differs in the view generation process, i.e., how it applies the reparameterization trick and where it disentangle the information. For detailed explanations, please refer to the Appendix~\ref{sec:distinction-to-sota}.

At first glance, our method may seem to resemble InfoTS~\citep{luo2023time}, since it also leverages the same reparameterization trick to facilitate the view generation. However, their $p_i$ indicates the probability of sampling a predefined transformation$ \mathcal{T}_i(\cdot)$, i.e., $\mathcal{A}_{\text{InfoTS}}(x) = \frac{1}{m} \sum_{i=1}^{m}\text{Gumbel-Softmax}(p_i) \mathcal{T}_i(x)$.It fails to handle the noise and artifacts introduced by predefined augmentations $\mathcal{T}_i(\cdot)$. On the contrary, our approach elegantly eliminates the dependency on $\mathcal{T}_i(\cdot)$ by preserving critical elements and modifying the noise elements in the frequency domain. This more effectively enables preserving contextual-related information in the generated views while infusing variance. 
FreRA also appears similar to AutoTCL~\citep{DBLP:journals/corr/abs-2402-10434} in the sense that it disentangles the informative information of the time series from the noisy ones. However, performing the disentanglement on the time domain disrupts the periodicity and inter-dependencies among timestamps in the real world and hinders the semantics from the input space. Conversely, we disentangle the information in the frequency domain and leverage its advantages over the time domain: global, independent, and compact, to better facilitate the view generation. 

\section{Experiments}

% \begin{table}[h]
% \centering
% \scriptsize
% \begin{tabular}{c|cccc}
% \toprule
% Dataset               & N      & Channel & Length & Class \\\midrule
% UCIHAR                & 10,299   & 9       & 128    & 6     \\
% % SHAR                  & 5,256   & 3       & 151    & 17    \\
% MS           & 13,592   & 12      & 200    & 6     \\
% WISDM                 & 7,089   & 3       & 200    & 6     \\
% FM       & 416    & 28      & 50     & 2     \\
% FD         & 9,414   & 144     & 62     & 2     \\
% HMD & 234    & 10      & 400    & 4     \\
% Heartbeat             & 409    & 61      & 405    & 2     \\
% Libras                & 360    & 2       & 45     & 15    \\
% \bottomrule
% \end{tabular}
% \caption{Details of 8 datasets. N stands for the number of samples in the dataset. }
% \label{Table:datasets}
% \end{table}

\noindent\textbf{Datasets.}
% \paragraph{Datasets}
To fully evaluate the model performance under different scenarios, we conduct extensive experiments on: (1) 3 large-scale datasets on HAR: UCIHAR~\citep{anguita2012human}, MotionSense (MS)~\citep{DBLP:conf/iotdi/MalekzadehCCH19}, and WISDM~\citep{DBLP:journals/sigkdd/KwapiszWM10}; (2) the UEA archive~\citep{bagnall2018uea}: 30 multivariate time series datasets from various applications such as Human Activity Recognition (HAR), Motion classification, ECG classification, EEG/MEG classification, Audio Spectra Classification and so on; (3) the UCR archive~\citep{dau2019ucr}: 100 univariate time series datasets collected from real-world scenarios; (4) a large-scale anomaly detection dataset: Fault Diagnosis (FD)~\citep{lessmeier2016condition} aiming to detect and classify bearing damages from single-channel current signals of electric motors; (5) a large-scale HAR dataset for transfer learning scenario: SHAR~\citep{DBLP:journals/corr/MicucciMN16}, which contains daily activity signals from 30 persons and is empirically observed to have large distribution gap among individuals~\citep{DBLP:conf/kdd/QianTM22}.
% 8 size-varying multivariate time series datasets among three applications: (1) \textbf{Human Activity Recognition}: 
% % SHAR~\citep{DBLP:journals/corr/MicucciMN16}, 
% UCIHAR~\citep{anguita2012human}, MotionSense (MS)~\citep{DBLP:conf/iotdi/MalekzadehCCH19}, WISDM~\citep{DBLP:journals/sigkdd/KwapiszWM10}, Libras; (2) \textbf{EEG Signal Classification}: FingerMovements (FM), FaceDetection (FD), HandMoveDirection (HMD); (3) \textbf{Audio Classification}: Heartbeat.   
% Note that Libras, FingerMovements, FaceDetection, HandMoveDirection, and Heartbeat are from the UEA archive~\citep{DBLP:journals/corr/abs-1811-00075}. Details of all the datasets are listed in Table~\ref{Table:datasets}.

\noindent\textbf{Baselines.}
% \paragraph{Baselines}
We compare FreRA against the following related baselines:
(1) 11 commonly-used handcrafted time-domain (T) augmentations~\citep{DBLP:conf/kdd/QianTM22}, including jitter, scaling, negation, permutation, shuffling, time-flipping, time-warping, resampling, rotation, permutation-and-jitter, jitter-and-scale;
(2) 5 handcrafted frequency-domain (F) augmentations~\citep{DBLP:conf/kdd/QianTM22}, including low-pass filter, high-pass filter, phase shift, amplitude and phase perturbation (fully), and amplitude and phase perturbation (partially);
(3) 3 SOTA automatic augmentation for contrastive learning: InfoMin~\citep{DBLP:conf/nips/Tian0PKSI20}, InfoTS~\citep{luo2023time}, and AutoTCL~\citep{DBLP:journals/corr/abs-2402-10434};
(4) 5 SOTA time series contrastive learning frameworks: TS2Vec~\citep{yue2022ts2vec}, TNC~\citep{DBLP:conf/iclr/TonekaboniEG21}, TS-TCC~\citep{DBLP:journals/pami/EldeleRCWKLG23}, TF-C~\citep{zhang2022self} and SoftCLT~\citep{lee2023soft}.
    % \item InfoTS~\citep{DBLP:conf/aaai/LuoCWXNYZLCC023}: an information-aware automatic augmentation for time series contrastive learning by learning a sampling distribution for a set of predefined transformations. 
    % \item InfoMin~\citep{DBLP:conf/nips/Tian0PKSI20}: an adversarial training objective looking for strong augmentations. It trains the contrastive encoder to maximize $\mathcal{L}_{\text{CL}}$ and adversarially trains the view generator to minimize $\mathcal{L}_{\text{CL}}$. Empirically, to make it suitable for time series, we use our critical frequency component identifier $\mathbf{p}$ to substitute the flow-based view generator which is designed for images. This implementation makes it benefit from our frequency-enhanced approach. We denote this baseline as InfoMin$^+$.
% \end{itemize}

\noindent\textbf{Implementation details.}
% \paragraph{Implementation Details}
We use the predefined train-validation-test split if the dataset includes such information. Otherwise, we split each dataset with a ratio of 64\%:16\%:20\%. For time-series classification datasets with class imbalance issues, we sample training instances with probabilities inversely proportional to their class sizes. We implement FreRA in PyTorch~\citep{paszke2019pytorch} and conduct all experiments on an NVIDIA GeForce RTX 3090 GPU with 25 GB memory. Additional implementation details are included in Appendix~\ref{sec:implementation}.

\begin{table*}[!ht]
\caption{Performance on anomaly detection task on the Fault Diagnosis dataset. Each row corresponds to a setting where the pre-training set includes domains $\{\text{a, bd, c}\} \setminus \text{Target Domain}$, and the Target Domain is used for evaluation. 
% Classification performance in transfer learning setting on the Fault Diagnosis dataset under different settings. 
% No. SD denotes the number of source domains for pre-training and TD denotes the index of the target domain. 
The best accuracy is highlighted in \textbf{bold}, and the second-best performance is {\ul underlined}.}
\label{Table:anomaly}
\centering
% \scriptsize
\footnotesize
% \resizebox{\textwidth}{!}{
\begin{tabular}{c|c|ccccccccccc}
\toprule
\begin{tabular}[c]{@{}c@{}}Target\\  Domain\end{tabular}       & Metrics  & \begin{tabular}[c]{@{}c@{}}FreRA\\  \textbf{(ours)}\end{tabular}          & best(T)        & best(F) & InfoMin$^+$        & InfoTS & AutoTCL & TS2Vec & TNC   & TS-TCC & TF-C & SoftCLT  \\\midrule
\multirow{2}{*}{a}  & ACC      & \textbf{0.620} & 0.574          & 0.519   & {\ul0.613} & 0.461  & 0.496   & 0.468  & 0.440 & 0.296  & 0.455 & 0.608\\
                    & Macro-F1 & \textbf{0.671} & 0.638          & 0.508   & {\ul0.644} & 0.485  & 0.484   & 0.468  & 0.302 & 0.353  & 0.208 & 0.639\\\midrule
\multirow{2}{*}{bd} & ACC      & \textbf{0.859} & {\ul0.826} & 0.767   & 0.807          & 0.731  & 0.433   & 0.802  & 0.455 & 0.823  & 0.455 & 0.808\\
                    & Macro-F1 & \textbf{0.895} & {\ul0.856} & 0.817   & 0.853          & 0.798  & 0.471   & 0.848  & 0.300 & 0.755  & 0.208 & 0.853\\\midrule
\multirow{2}{*}{c}  & ACC      & \textbf{0.819} & 0.810          & 0.736   & {\ul0.812} & 0.742  & 0.482   & 0.677  & 0.465 & 0.557  & 0.455 & 0.775\\
                    & Macro-F1 & \textbf{0.858} & 0.794          & 0.755   & {\ul0.848} & 0.781  & 0.456   & 0.747  & 0.314 & 0.617  & 0.208 & 0.825
                    \\\bottomrule
\end{tabular}
% }
% \vspace{-10pt}
\end{table*}

\begin{table*}[!t]
\caption{Classification performance in transfer learning setting on the SHAR dataset under different numbers of source domains. No. SD denotes the number of source domains for pre-training and TD denotes the index of the target domain. The best accuracy is highlighted in \textbf{bold}, and the second-best performance is {\ul underlined}.}
\label{Table:transfer}
\centering
% \tiny
\footnotesize
% \small
% \scriptsize
% \resizebox{\textwidth}{!}{
\begin{tabular}{c|c|ccccccccccc}
\toprule
\begin{tabular}[c]{@{}c@{}}No.\\  SD\end{tabular}   & TD & \begin{tabular}[c]{@{}c@{}}FreRA\\  \textbf{(ours)}\end{tabular}          & best(T) & best(F) & InfoMin$^+$ & InfoTS & AutoTCL & TS2Vec & TNC   & TS-TCC & TF-C & SoftCLT  \\ \midrule
\multirow{4}{*}{3}  & 1             & \textbf{0.602} & {\ul 0.599}   & 0.495   & 0.537   & 0.367  & 0.464   & 0.430  & 0.133 & 0.495  & 0.349 & 0.505\\
                    & 2             & \textbf{0.467} & {\ul 0.415}   & 0.412   & 0.359   & 0.369  & 0.278   & 0.317  & 0.145 & 0.410  & 0.252 & 0.407\\
                    & 3             & \textbf{0.665} & 0.582   & {\ul 0.599}   & 0.516   & 0.516  & 0.414   & 0.523  & 0.217 & 0.464  & 0.568 & 0.530\\
                    & 5             & \textbf{0.366} & 0.332   & 0.336   & 0.359   & 0.081  & 0.245   & 0.050  & 0.143 & {\ul 0.362}  & 0.255 & 0.339\\ \midrule
\multirow{4}{*}{19} & 1             & \textbf{0.628} & 0.555   & {\ul 0.607}   & 0.542   & 0.599  & 0.497   & 0.568  & 0.117 & 0.578  & 0.453 & 0.581\\
                    & 2             & \textbf{0.652} & 0.583   & 0.571   & 0.563   & 0.455  & 0.372   & 0.640  & 0.148 & {\ul 0.647}  & 0.456 & 0.581\\
                    & 3             & \textbf{0.691} & 0.628   & {\ul 0.665}   & 0.638   & 0.563  & 0.408   & 0.502  & 0.135 & 0.592  & 0.451 & 0.559\\
                    & 5             & \textbf{0.698} & 0.617   & 0.638   & 0.601   & 0.638  & 0.430   & {\ul 0.658}  & 0.204 & 0.612  & 0.466 & 0.567\\ \bottomrule
\end{tabular}
% }
% \vspace{-5pt}
\end{table*}

\subsection{Main Results on Time Series Classification Tasks}

The overall results on all the datasets are presented in Table~\ref{Table:result-overall}. Overall, FreRA consistently outperforms all the baselines on the three large HAR datasets and achieves the top average accuracy and ranking on both UEA and UCR archives. The detailed performances of the UEA and UCR archives are reported in Table~\ref{Table:uea} and Table~\ref{Table:ucr1} in the Appendix. The critical difference diagrams on UEA and UCR archives are presented in Figure~\ref{fig:cdd} in the Appendix. FreRA achieves the best performance on 17 out of 30 datasets in the UEA archive.
% and the improvement over the second-best performance is up to $1.2\%$. 
We credit the surprising performance to the frequency-refined views generated by FreRA. The empirical performance adequately illustrates that FreRA can effectively keep the semantically relevant information from critical frequency components intact while infusing variance, boosting representation learning on all datasets. FreRA achieves leading performances not only on large-scale HAR datasets but also on extremely small datasets, e.g., AtrialFibrillation, DuckDuckGeese, and StandWalkJump within the UEA archive, whose training sets contain less than 100 samples, and the improvement over the second-best baselines is up to $8.7\%$ on average. This is not only credited to the effectiveness of FreRA in maintaining semantics but also to the lightweight and scalable design where the number of parameters is only half of the sequence length. It also indicates that FreRA provides robust performance across datasets of varying sizes. Although FreRA achieves an average ranking 0.01 lower than InfoTS on the UCR archive, when comparing FreRA to the baselines of the same backbone structure, i.e. best(T), best(F) and InfoTS$^+$, the improvement brought by the augmentation itself is significantly larger than the difference between InfoTS and TS2Vec. This indicates that FreRA offers stronger enhancement regardless of the backbone. We present a detailed analysis in the Appendix~\ref{sec:main_analysis}. 

\subsection{Evaluation on Anomaly Detection Tasks}
\label{sec:anomaly}
We evaluate the performance of FreRA on the anomaly detection task using the Fault Diagnosis dataset and present the results in Table~\ref{Table:anomaly}. The signals are collected under 4 different operation settings \{a, b, c, d\}. Observing the negligible domain gap between signals from settings `b' and `d', we randomly sample half of the data from each setting and combine them as a new domain `bd'. 
% Each signal is classified as healthy, inner ring damage, or outer ring damage. 
Considering the highly imbalanced class distribution, we include the Macro-F1 score as another evaluation metric. FreRA outperforms all the baselines on both evaluation metrics, which demonstrates its strong performance in applications beyond classification.

\subsection{Evaluation on Transfer Learning}
% It is not always the case that the per-training data and data from downstream tasks are well aligned as earlier experiments. In practice, the distribution of time series data varies among environment setups and sources. 
Here, we evaluate the generalizability of the pre-trained encoder, which is crucial when there exists a misalignment between the per-training data and data from downstream tasks. The encoder is pre-trained on the source domains and adopted directly to an unseen target domain. Following ~\citep{DBLP:conf/kdd/QianTM22}, we conduct transfer learning in data-scarce and data-rich settings, where the number of source domains for training is 3 and 19 respectively. Table~\ref{Table:transfer} records the results from the two settings. FreRA is shown to be more effective in learning generalizable encoders than all the baselines. This is because the views emphasizing the semantic-preserving patterns guide the training of the encoder and make it sensitive to the inherent global semantic information and robust to the \revise{semantically irrelevant} information, i.e., distribution shift among different domains. Without effectively identifying critical and unimportant information, other SOTA baselines on automatic augmentation and time series contrastive learning fail to deliver promising performance in the transfer learning scenario.

\begin{table*}[!t]
\caption{The performance of the three large HAR datasets on alternative contrastive learning models. }
\label{Table:result-clmodel1}
\centering
% \footnotesize
\small
% \tiny
% \scriptsize
\resizebox{1.0\textwidth}{!}{
\begin{tabular}{c|cc|cc|cc|ccc|ccc|ccc}
\toprule
                          & \multicolumn{2}{c}{TS2Vec (InfoNCE)}   & \multicolumn{2}{c}{TS-TCC (InfoNCE)}   & \multicolumn{2}{c}{{ SoftCLT (InfoNCE)}}  & \multicolumn{3}{c}{{SimCLR (InfoNCE)}} & \multicolumn{3}{c}{{ SimCLR (NT-Xent)}}                                             & \multicolumn{3}{c}{BYOL (Cosine Similarity)}                                               \\
\multirow{-2}{*}{Dataset} & \begin{tabular}[c]{@{}c@{}}FreRA\\  \textbf{(ours)}\end{tabular}          & \begin{tabular}[c]{@{}c@{}}original\\  augmentation\end{tabular} & \begin{tabular}[c]{@{}c@{}}FreRA\\  \textbf{(ours)}\end{tabular}          & \begin{tabular}[c]{@{}c@{}}original\\  augmentation\end{tabular} & { \begin{tabular}[c]{@{}c@{}}FreRA\\  \textbf{(ours)}\end{tabular}}          & { \begin{tabular}[c]{@{}c@{}}original\\  augmentation\end{tabular}} 
 & { \begin{tabular}[c]{@{}c@{}}FreRA\\  \textbf{(ours)}\end{tabular}}          & { best(T)} & { best(F)} & { \begin{tabular}[c]{@{}c@{}}FreRA\\  \textbf{(ours)}\end{tabular}}          & { best(T)} & { best(F)} & \begin{tabular}[c]{@{}c@{}}FreRA\\  \textbf{(ours)}\end{tabular}              & best(T)    & best(F)\\\midrule

UCIHAR                    & \textbf{0.970} & 0.959                 & \textbf{0.944} & 0.924                 & { \textbf{0.969}} & { 0.961} & {\textbf{0.975}}  & {0.959}    & {0.960}    & { \textbf{0.972}} & { 0.951}   & { 0.955}   & \textbf{0.960}     & 0.940      & 0.937                                         \\
MS                        & \textbf{0.968} & 0.945                 & \textbf{0.959} & 0.915                 & { \textbf{0.974}} & { 0.962} & {\textbf{0.982}}  & {0.956}    & {0.970}    & { \textbf{0.979}} & { 0.969}   & { 0.965}   & \textbf{0.983}     & 0.968      & 0.954                                        \\
WISDM                     & \textbf{0.957} & 0.939                 & \textbf{0.962} & 0.889                 & { \textbf{0.956}} & { 0.952} & {\textbf{0.972}}  & {0.942}    & {0.950}    & { \textbf{0.966}} & { 0.941}   & { 0.952}   & \textbf{0.952}     & 0.942      & 0.928                      \\\bottomrule                  
\end{tabular}
}
\end{table*}

\begin{table}[!t]
\caption{Effects of the two modification modules and the L1-norm regularization of FreRA. Results are averaged over 30 datasets from the UEA archive. The number in the bracket illustrates the accuracy gap with FreRA.}
\label{Table:result-component}
\centering
% \tiny
% \scriptsize
% \footnotesize
\small
% \resizebox{0.46\textwidth}{!}{
\begin{tabular}{l|c}
\toprule
Methods                                & Avg. ACC       \\ \midrule
FreRA \textbf{(ours)}                                 & \textbf{0.754} \\ \hline
\begin{tabular}[l]{@{}l@{}}w/o semantic-aware identity modification \\ on critical components\end{tabular} & 0.690 (-0.064) \\\hline
\begin{tabular}[l]{@{}l@{}}w/o semantic-agnostic self-adaptive modification \\ on noise components\end{tabular}    & 0.695 (-0.059) \\\hline
% \begin{tabular}[l]{@{}l@{}}w/o L1-norm regularization \\ on $\mathbf{w}_\text{crit}$\end{tabular}
w/o L1-norm regularization on $\mathbf{w}_\text{crit}$ & 0.690 (-0.064) \\
\bottomrule 
\end{tabular}
% }
% \vspace{-5pt}
\end{table}

\subsection{Ablation Studies}

% \paragraph{Ablation Study on Each Modulation.}
\noindent\textbf{Effect of each component.}
In Table~\ref{Table:result-component}, we evaluate the effect of each component of FreRA, i.e. the \revise{semantic-aware} identity modification on critical frequency components, the \revise{semantic-agnostic} self-adaptive modification on unimportant frequency components, and the regularization term. 
% \revise{In addition, we evaluate the performance with all three components removed. It downgrades the model to a vanilla contrastive learning framework where both views are the original input. }
To disallow the identity modification on critical components, we randomly sample a proportion of critical components instead of identifying their distribution in a data-driven way. The proportion is the same as $\mathcal{L}_{\text{reg}}$ from the last epoch of our approach to ensure fair comparison. 
% we randomly select a number of Fourier components to construct the Fourier domain of the augmented view and let the number equal to the value of $\lceil \mathcal{L}_{\text{sparse}} \rceil$ from the last epoch of our approach to ensure fair comparison. 
To disallow the self-adaptive modification on unimportant frequency components, we set $\mathbf{w}_\text{dist}$ as an all-zero vector. 
To ignore the regularization term, we let hyper-parameter $\lambda$ be 0. 
Overall, removing any component deteriorates performance drastically. The semantic information and the distortion introduced are both crucial for downstream tasks. Between the two, the identity modification on critical components is slightly more important than the distortion on noisy components. It demonstrates the effectiveness of isolating critical and non-critical components from the frequency domain and applying the respective modifications accordingly.
Last but not least, the L1-norm regularization is as crucial as the two frequency modification modules. The result demonstrates the importance of maintaining the inherent compact distribution of critical components.

\begin{figure}[!t]
    \centering
\includegraphics[width=0.46\textwidth]{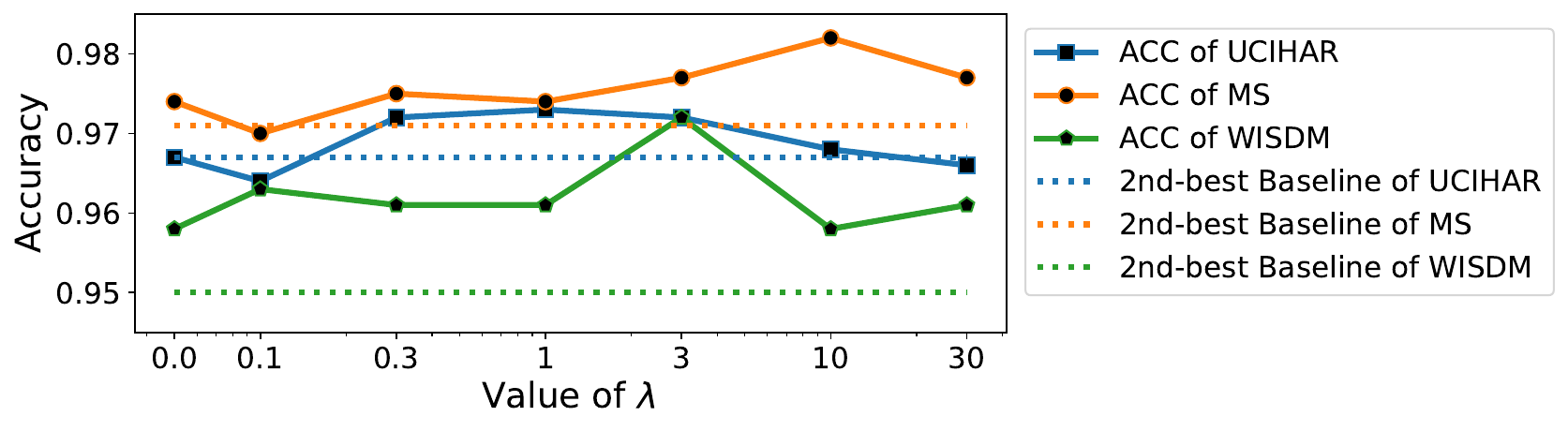}
    \caption{Performance of FreRA on the 3 HAR datasets under varying $\lambda$, in comparison to their second-best baselines. }
    \label{fig:param}
    % \vspace{-10pt}
\end{figure}

\noindent\textbf{Sensitivity to Hyper-parameter $\lambda$.}
% \textbf{Sensitivity to Hyper-parameter $\lambda$.
% quantile k: one large dataset one small dataset
% Figure~\ref{fig:param} in the Appendix shows the accuracy of FreRA on the 3 HAR datasets under varying $\lambda$ compared to their second-best baselines plotted in dashed lines for reference. The result demonstrates that the downstream task performance remains stable across different values of $\lambda$ and consistently better than the baseline, indicating FreRA is robust to the selection of the hyper-parameter's value. A detailed analysis is presented in the Appendix~\ref{sec:ablation}.
Figure~\ref{fig:param} shows the accuracy of FreRA on the 3 HAR datasets under varying $\lambda$ compared to their second-best baselines plotted in dashed lines for reference. The result demonstrates that the downstream task performance remains stable across different values of $\lambda$ and consistently better than the baseline, indicating FreRA is robust to the selection of the hyper-parameter's value. 
In Figure~\ref{fig:param}, UCIHAR, MS and WISDM achieve peak performances at $\lambda = 1, 10, 3$ respectively. On the left of the peak, the performance is suboptimal because redundant frequency components are included in the critical components. On the right of the peaks, incomplete critical information leads to degraded performances. The peak values indicate that FreRA learns to preserve only \revise{semantically relevant} frequency components and distort the \revise{semantically irrelevant} components, achieving the optimal view for representation learning.  

\noindent\textbf{On Alternative Contrastive Learning Frameworks.}
% Different from InfoTS, which is built on the TS2Vec~\citep{DBLP:conf/aaai/YueWDYHTX22} framework and does not suit other contrastive learning models, 
Due to its meticulous design, FreRA can be seamlessly integrated with different contrastive models in a plug-and-play manner.
% We exclude InfoMin as a baseline in this experiment because BYOL and SimSiam do not apply MI bound as the objective function, which is required for InfoMin.
In Table~\ref{Table:result-clmodel1}, we apply FreRA to five alternative contrastive learning models: (1) three time-series contrastive models TS2Vec~\citep{DBLP:conf/aaai/YueWDYHTX22}, TS-TCC~\citep{DBLP:journals/pami/EldeleRCWKLG23}, and SoftCLT~\citep{lee2023soft} with their default augmentations as baselines and (2) two general purpose contrastive learning models originally designed for the vision domain,  BYOL~\citep{grill2020bootstrap} and SimCLR~\citep{DBLP:conf/icml/ChenK0H20} with the best time-domain and frequency-domain augmentations as baselines. For SimCLR, despite the NT-Xent loss originally applied in SimCLR, we also use InfoNCE as the loss function, which forms the framework we use in our main result. The same usage has been deployed in \citet{yeh2022decoupled} and \citet{wu2024understanding} as well. Our current evaluation covers 5 contrastive learning frameworks and 3 types of contrastive loss functions. All the models differ in network design and optimization objectives. It is worth noting that the contrastive losses used in TS-TCC, TS2Vec, and SoftCLT are different variants of InfoNCE, each with its unique formulation. The results presented consistently demonstrate that FreRA is a plug-and-play method that effectively enhances existing contrastive learning frameworks. This experiment highlights the flexibility and adaptability of our approach.

% \paragraph{More Ablation Studies.}
More comprehensive ablation studies investigating the effect of unimportant component selection mechanisms, and a case study on the learned vector $\mathbf{s}$ 
% , and the robustness to Gaussian noise 
are presented in the Appendix~\ref{sec:ablation}.

\section{Conclusion}
In this paper, we propose Frequency-Refined Augmentation (FreRA), a lightweight yet effective automatic augmentation designed for time series contrastive learning on classification tasks. FreRA leverages the global, independent, and compact properties of the frequency domain to generate semantic-preserving views through independent modifications on separated frequency components. 
% Specifically, it applies identity modulation on critical components to preserve semantics and self-adaptive modulation on unimportant components to introduce variance. 
% Component separation and modulation are guided by a single trainable vector, which learns the semantic distribution in the frequency domain, with a size only half the sequence length. 
Its effectiveness is verified both theoretically and empirically. Experiments on 135 benchmark datasets across diverse applications demonstrate that FreRA is universally effective in contrastive learning and generalizes well in transfer learning scenarios. In addition, it is robust to hyper-parameter settings, flexible and effective when applied to various contrastive learning frameworks
% , resilient to Gaussian noise added to the input
, and consistently able to capture the inherent semantically relevant information. 
% Lastly, we discuss the limitations of our work. First, FreRA only functions as an improved augmentation. Whether any modification in the backbone can cooperate with FreRA to perform better representation still remains unexplored. Second, we have not conducted evaluations in other settings such as continual learning, leaving the relevant improvement unaddressed.

\begin{acks}
 This research is supported by the RIE2025 Industry Alignment Fund – Industry Collaboration Projects (IAF-ICP) (Award I2301E0026), administered by A*STAR, as well as supported by Alibaba Group and NTU Singapore through Alibaba-NTU Global e-Sustainability CorpLab (ANGEL). Dr. Qian Hangwei is supported by A*STAR Career Development Fund <Project No. C243512010>.
\end{acks}


\clearpage

\bibliographystyle{ACM-Reference-Format}
\bibliography{sample-base}


\appendix

\section{Appendix}

\subsection{Details regarding Figure~\ref{fig:MI}}
\label{sec:explain-fig1}
Directly calculating the mutual information between the entire time series and the label is not trivial due to the curse of dimensionality. To address this, in Figure~\ref{fig:MI}, we calculate the mutual information between each timestamp and the label. It visualizes the amount of semantic information preserved across all the timestamps.
For instance, the value of the orange curve at timestamp $\forall t \in [1,..., L]$ is $\text{MI}(\textsf{x}^t; \textsf{y})$ of the UCIHAR dataset, where $\textsf{x}^t \in \mathbb{R}^D$ is the signal at timestamp $t$, $\textsf{y}$ is the ground truth label. $L=128$ and $D=6$ are the length and dimension of the samples. Estimating the timestamp-wise mutual information, with a dimension of only 6, allows us to avoid the curse of dimensionality.
% Intuitively, a sliding window should have a higher MI with the label compared to a single timestamp. However, it does not mean a single timestamp contains zero information about the label, as indicated by the non-zero values in our figure. 
In this plot, we do not intend to suggest a single timestamp alone is fully representative of the underlying semantics. Instead, the figure illustrates how the informative content varies across different augmentation functions. While a single timestamp may not directly indicate specific semantic meaning, the plot demonstrates the manipulation of the frequency domain benefits the augmented views. This is attributed to the undeteriorated critical components that are semantically informative. 

% \subsection{Discussion on Predefined Frequency-domain Augmentations}

% Frequency-based predefined augmentations, such as high-pass and low-pass filters, require prior knowledge, such as the effective bandwidth of the dataset, to determine the selection of appropriate augmentation functions. Additionally, stochastic frequency-domain augmentations, such as phase-shift and augmentation in TF-C~\cite{zhang2022self}, introduce random noise that disrupts the critical information. As prior knowledge is not always accessible in the contrastive learning paradigm, and the compromised semantics caused by the frequency-domain augmentation have been observed from the brown line in Figure 1, predefined frequency-domain augmentations remain suboptimal for contrastive learning. Existing frequency-based augmentations do not fully leverage the advantages of the frequency domain. To this end, we reanalyze the benefits of the frequency domain and deliberately design a frequency-based augmentation to address the aforementioned issues and fully utilize its advantages. 

\subsection{Properties of the Discrete Fourier Transform (DFT)}

\subsubsection{Conjugate Symmetric.}
\label{sec:conj-sym}
\sloppy Given a signal $x(n) \in \mathbb{R}, n \in \{0, 1, ..., L-1\}$, its DFT $X(m) \in \mathbb{C}, m \in \{0, 1, ..., L-1\}$ is conjugate symmetric, i.e., $X(L-m) = \overline{X(m)}$. The proof is as follows:
\begin{equation}
\begin{split}
    X(L-m) & = \sum_{n=0}^{L-1} x(n) e^{-\frac{2 \pi i}{L} (L-m) n} \\
    & = \sum_{n=0}^{L-1} x(n) e^{\frac{2 \pi i}{L} m n} \\
    & = \overline{X(m)}.
\end{split}
\end{equation}
Converting back to our notation, $X(m) = \overline{X(L-m)}=x_f^{L-m+1}$, which explains the second condition of $X(m)$ in Eq.~\ref{Equation:annotation-conv}. The Conjugate Symmetry allows only half of the DFT signal to recover the entire time series, which also justifies why FreRA manipulates only half of the frequency components, i.e., $F=\lfloor{L/2}\rfloor +1$. This property allows FreRA to have a lightweight structure. 

\subsubsection{Orthogonal of Fourier basis.}
\label{sec:orthogonal}
The inverse DFT, i.e., 
\begin{equation}
    x(n)=\frac{1}{L}\sum_{m=0}^{L-1} X(m) e^{\frac{2 \pi i}{L} m n}, 
\end{equation}
% $x(n)=\frac{1}{L}\sum_{m=0}^{L-1} X(m) e^{\frac{2 \pi i}{L} m n}$ 
decompose the time domain on the Fourier basis $\mathbf{u}_m = [e^{\frac{2 \pi i}{L} m n} | n=0, 1, ..., L-1]^T \in \mathbb{C}^L$, where frequency components $X(m)$ are the coefficients with respect to the Fourier basis. 
The orthogonal property of Fourier basis, i.e., $\langle \mathbf{u}_m, \mathbf{u}_q \rangle = 0$ if $m \neq q$, is proved below.

\begin{equation}
\begin{split}
    \langle \mathbf{u}_m, \mathbf{u}_q \rangle & = \mathbf{u}_m^T \overline{\mathbf{u}_q} \\
    & = \sum_{n=0}^{L-1} e^{\frac{2 \pi i}{L} m n} e^{-\frac{2 \pi i}{L} q n} \\
    & = \sum_{n=0}^{L-1} e^{\frac{2 \pi i}{L} (m-q) n} \\
     (\text{the sum of } & \text{a geometric series follows: } \sum_{n=0}^{L-1} r^n = \frac{1-r^L}{1-r})\\
    & = \frac{1-e^{\frac{2 \pi i}{L} L (m-q)}}{1-e^{\frac{2 \pi i}{L} (m-q)}} \\
    (e^{\frac{2 \pi i}{L} L (m-q)} & = 1 \text{ and } e^{\frac{2 \pi i}{L} (m-q)} \neq 1 \text{ if } m \neq q) \\
    & = 0
\end{split}
\end{equation}

\subsection{Proofs of Propositions}

\begin{proposition}{\textbf{(Conservation of Entropy)}}
Let $\textsf{x}$ and $\textsf{x}_f$ be the random variables denoting the time series in the time domain and the frequency domain respectively, then we have $\text{H}(\textsf{x})=\text{H}(\textsf{x}_f)$. 
\label{proposition:cons-entropy}
\end{proposition}

\begin{proof}
Since the DFT is a one-to-one invertible transformation, we have $p(x) = p(x_f)$.
\begin{equation}
    \begin{split}
    \text{H}(\textsf{x}) & = \sum_{x} p(x) \log p(x) \\ 
    & = \sum_{x_f} p(x_f) \log p(x_f) \\
    & = \text{H}(\textsf{x}_f)
    \end{split}
\end{equation}
\end{proof}

\begin{proposition}{\textbf{(Conservation of Mutual Information)}}
Let $\textsf{x}$, $\textsf{x}_f$, and $\textsf{y}$ be the random variables denoting the time series in the time domain and the frequency domain, and their corresponding label respectively, then we have $\text{MI}(\textsf{x}; \textsf{y})=\text{MI}(\textsf{x}_f; \textsf{y})$. 
\label{proposition:cons-MI}
\end{proposition}

\begin{proof}
\label{proof:cons-MI}
Since the DFT does not alter the label of the time series variable, we have $p(x, y) = p(x_f, y)$.
\begin{equation}
    \begin{split}
    \text{MI}(\textsf{x}; \textsf{y}) & = \text{H}(\textsf{y}) - \text{H}(\textsf{y} | \textsf{x}) \\
    & = \text{H}(\textsf{y}) - \sum_{x, y} p(x, y) \log {\frac{p(x, y)}{p(x)}} \\ 
    & = \text{H}(\textsf{y}) - \sum_{x_f, y} p(x_f, y) \log {\frac{p(x_f, y)}{p(x_f)}} \\
    & = \text{MI}(\textsf{x}_f; \textsf{y})
    \end{split}
\end{equation}
Similarly, we can proof $\text{MI}(\textsf{x}; \tilde{\textsf{x}}) = \text{MI}(\textsf{x}_f; \tilde{\textsf{x}}_f)$, where random variable $\textsf{x}$ and $\tilde{\textsf{x}}$ denotes two time series and $\textsf{x}_f$ and $\tilde{\textsf{x}}_f$ denotes their frequency-domain counterpart. 
\end{proof}

\begin{proposition}{\textit{With the reliable assumption that the noisy frequency components are independent to the label, FreRA is a semantic preserving transformation, i.e., $\text{MI}(\mathcal{A}_\mathbf{s}(\textsf{x}) ; \textsf{y}) = \text{MI}(\textsf{x}; \textsf{y})$.}}
\label{proposition:cri_sem}
\end{proposition}

\begin{proof}
\label{proof:semantic-preserving}
    Let $\textsf{x}_f^\text{crit} = \mathbf{w}_\text{crit} \odot \textsf{x}_f$ and $\textsf{x}_f^\text{dist} =( \mathbf{1} - \mathbf{w}_\text{crit}) \odot \textsf{x}_f$ denote the critical and noisy frequency components respectively.
    Knowing $\textsf{x}_f^\text{crit}$ and $\textsf{x}_f^\text{dist}$ are independent, we have 
    \begin{equation}
        \text{H}(\textsf{x}_f) = \text{H}(\textsf{x}_f^\text{crit}) + \text{H}(\textsf{x}_f^\text{dist}).
    \end{equation}
    Then we show that 
    \begin{equation}
    \begin{split}
        \text{MI}(\textsf{x}_f; \textsf{y}) & = \text{H}(\textsf{x}_f) - \text{H}(\textsf{x}_f; \textsf{y}) \\
        & = \text{H}(\textsf{x}_f^\text{crit}) + \text{H}(\textsf{x}_f^\text{dist}) - \text{H}(\textsf{x}_f^\text{crit}, \textsf{x}_f^\text{dist} | \textsf{y}) \\
        (\textsf{x}_f^\text{dist} & \text{ , as irrelevant components, is independent to } \textsf{y}) \\
        & = \text{H}(\textsf{x}_f^\text{crit}) + \text{H}(\textsf{x}_f^\text{dist}) - (\text{H}(\textsf{x}_f^\text{crit} | \textsf{y}) + \text{H}(\textsf{x}_f^\text{dist})) \\
        & = \text{H}(\textsf{x}_f^\text{crit}) - \text{H}(\textsf{x}_f^\text{crit} | \textsf{y}) \\
        & = \text{MI}(\textsf{x}_f^\text{crit}; \textsf{y}).
    \end{split}
    \end{equation}
    Similarly, 
    \begin{equation}
    \resizebox{0.46\textwidth}{!}{$
    \begin{split}
        \text{MI}(\mathcal{A}_\mathbf{s}(\textsf{x}); \textsf{y}) & = \text{MI}((\mathbf{w}_\text{crit} + \mathbf{w}_\text{dist}) \odot \textsf{x}_f ; \textsf{y}) \\
        & = \text{H}((\mathbf{w}_\text{crit} + \mathbf{w}_\text{dist}) \odot \textsf{x}_f) - \text{H}((\mathbf{w}_\text{crit} + \mathbf{w}_\text{dist}) \odot \textsf{x}_f; \textsf{y}) \\
        & = \text{H}(\mathbf{w}_\text{crit} \odot \textsf{x}_f) + \text{H}(\mathbf{w}_\text{dist} \odot \textsf{x}_f) \\
        & \quad - \text{H}((\mathbf{w}_\text{crit} \odot \textsf{x}_f + \mathbf{w}_\text{dist} \odot \textsf{x}_f | \textsf{y}) \\
        (\mathbf{w}_\text{dist} & \odot \textsf{x}_f \text{ is independent to } \textsf{y}) \\
        & = \text{H}(\textsf{x}_f^\text{crit}) + \text{H}(\mathbf{w}_\text{dist} \odot \textsf{x}_f) - (\text{H}(\textsf{x}_f^\text{crit} | \textsf{y}) + \text{H}(\mathbf{w}_\text{dist} \odot \textsf{x}_f)) \\
        & = \text{H}(\textsf{x}_f^\text{crit}) - \text{H}(\textsf{x}_f^\text{crit} | \textsf{y}) \\
        & = \text{MI}(\textsf{x}_f^\text{crit}; \textsf{y}).
    \end{split}
    $}
    \end{equation}
% The proof for $\text{MI}(\textsf{x}_f; \textsf{y})=\text{MI}((\mathbf{w}_\text{crit} + \mathbf{w}_\text{dist}) \odot \textsf{x}_f ; \textsf{y})$ is similar since the set of noisy frequency components with modification applied, is still independent to $\textsf{y}$. 
Applying Proposition.~\ref{proposition:cons-MI}, we have $\text{MI}(\mathcal{A}_\mathbf{s}(\textsf{x}); \textsf{y}) = \text{MI}(\textsf{x}; \textsf{y})$.
\end{proof}

\begin{proposition}{\textit{$\text{MI}(\mathbf{w}_\text{crit})\odot \textsf{x} ; \textsf{x})$ is monotonically increasing w.r.t the proportion of critical components.}}
\label{proposition:mono-increase}
\end{proposition}

\begin{proof}
\label{proof:mono-increase}
    \begin{equation}
    \resizebox{0.46\textwidth}{!}{$
    \begin{split}
        \text{MI}(\mathbf{w}_\text{crit} \odot \textsf{x}_f; \textsf{x}_f) & = \text{H}(\textsf{x}_f) - H(\textsf{x}_f | \mathbf{w}_\text{crit} \odot \textsf{x}_f) \\
        & = \text{H}(\textsf{x}_f) - H(\mathbf{w}_\text{crit} \odot \textsf{x}_f, (\mathbf{1} - \mathbf{w}_\text{crit}) \odot \textsf{x}_f | \mathbf{w}_\text{crit} \odot \textsf{x}_f) \\
        & \text{(}\mathbf{w}_\text{crit} \odot \textsf{x}_f \text{ and } (\mathbf{1} - \mathbf{w}_\text{crit}) \odot \textsf{x}_f \text{ are independent} \\
        & \text{since they lie on the orthogonal basis)} \\
        & = \text{H}(\textsf{x}_f) - H((\mathbf{1} - \mathbf{w}_\text{crit}) \odot \textsf{x}_f) \\
        & = \text{H}(\textsf{x}_f) - \sum_{i=1}^F \mathbbm{1}_{\{1-\mathbf{w}_\text{crit}^i=1\}} \text{H}(\textsf{x}_f^i) 
    \end{split}
    $}
    \end{equation}
Since the first term $\text{H}(\textsf{x}_f)$ is fixed, and the second term 
% $\sum_{i=1}^F \mathbbm{1}_{\{1-\mathbf{w}_\text{crit}^i=1\}} \text{H}(\textsf{x}_f^i)$ 
decreases as the proportion of critical components increases, we prove the monotonic increasing of $\text{MI}(\mathbf{w}_\text{crit} \odot \textsf{x}_f; \textsf{x}_f)$ w.r.t the proportion of critical components. As the proportion becomes 1, i.e., all the frequency components are identified as critical ones, $\text{MI}(\mathbf{w}_\text{crit} \odot \textsf{x}_f; \textsf{x}_f)=\text{H}(\textsf{x}_f)$, as we plot in Figure~\ref{fig:MI-crit}. 
\end{proof}

% \subsection{Distinction to Existing Automatic Augmentation for Time Series Contrastive Learning}
% \label{sec:distinction-to-sota}
% At first glance, our method may seem to resemble InfoTS~\citep{luo2023time}, since it also leverages the same reparameterization trick to facilitate the view generation. However, their $p_i$ indicates the probability of sampling a predefined transformation$ \mathcal{T}_i(\cdot)$, i.e., $\mathcal{A}_{\text{InfoTS}}(x) = \frac{1}{m} \sum_{i=1}^{m}\text{Gumbel-Softmax}(p_i) \mathcal{T}_i(x)$.It fails to handle the noise and artifacts introduced by predefined augmentations $\mathcal{T}_i(\cdot)$. On the contrary, our approach elegantly eliminates the dependency on $\mathcal{T}_i(\cdot)$ by preserving critical elements and modifying the noise elements in the frequency domain. This more effectively enables preserving contextual-related information in the generated views while infusing variance. 
% FreRA also appears similar to AutoTCL~\citep{DBLP:journals/corr/abs-2402-10434} in the sense that it disentangles the informative information of the time series from the noisy ones. However, performing the disentanglement on the time domain disrupts the periodicity and inter-dependencies among timestamps in the real world and hinders the semantics from the input space. Conversely, we disentangle the information in the frequency domain and leverage its advantages over the time domain: global, independent, and compact, to better facilitate the view generation. 

\subsection{Implementation Details}
\label{sec:implementation}
For the predefined time-domain and frequency-domain augmentations, we follow the parameter settings from ~\citep{DBLP:conf/kdd/QianTM22}. For the InfoMin baseline, we apply its adversarial objective to replace our regularization term. To make it suitable for time series, we use our frequency-domain refinement to substitute the flow-based view generator which is designed for images. This implementation makes it benefit from our frequency-enhanced approach and we denote this baseline as InfoMin$^+$. For other 
baselines, we adopt the results from ~\citep{DBLP:journals/corr/abs-2402-10434,lee2023soft} if they are available. Otherwise, we use the publicly available implementation and fine-tune the model as suggested in the original papers.

Fully-convolutional Network (FCN)~\citep{DBLP:conf/ijcnn/WangYO17} with an output dimension 128 is adopted as the encoder $f_\theta$. The batch size is selected from $\{256, 128, 64, 32, 16, 5\}$ according to the scale of the dataset, and the maximum training epoch is set to 200 for all the experiments. The learning rate is selected from $\{0.03, 0.01, 0.003, 0.001\}$. We adopt SGD optimizer to train the contrastive model and Adam optimizer for $\mathbf{s}$. 
For the hyper-parameter setting, we select discretization temperature $\tau_w$ from $\{0.1, 0.2\}$, and fix temperature coefficient $\tau$ to be 0.2. $\lambda$ is searched from $\{0.1, 0.3, 1, 3, 10, 30\}$. The projector $g_\phi$ is a two-layer MLP, with hidden and output dimensions 128.

To evaluate the performance, we employ the commonly used linear evaluation protocol. We first jointly train FreRA and the contrastive learning model, then we discard other components and keep only the pre-trained encoder $f_{\theta^\ast}$ frozen and train a linear classifier on top of it, as illustrated in the lower right corner of Figure~\ref{fig:FreRA-overview}.  For time series classification tasks, we record the best accuracy (ACC) as the evaluation metric. For anomaly detection tasks, we record both the best accuracy and the Macro-F1 score.

\begin{figure}[!t]
   \begin{minipage}{0.46\textwidth}
     \centering
     \includegraphics[width=0.9\linewidth]{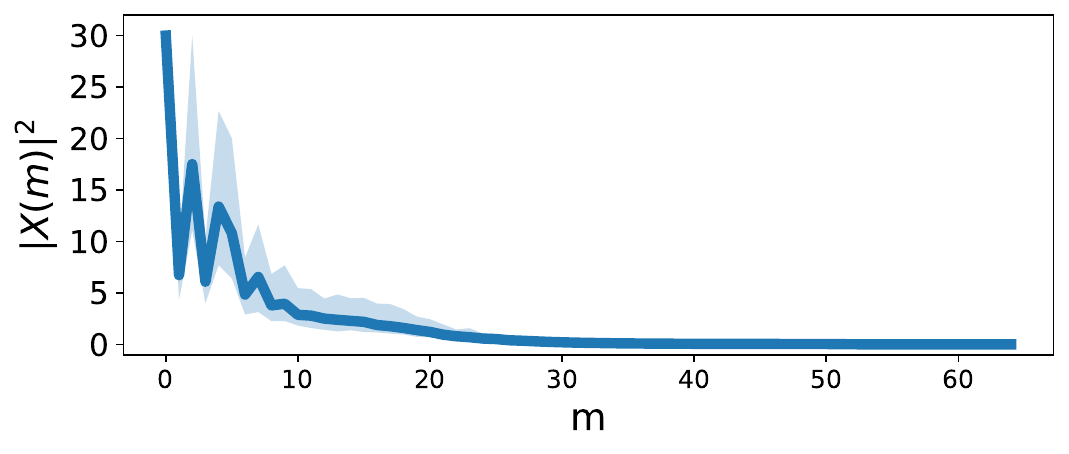}
     \caption{Take the UCIHAR dataset as an example, the energy in the frequency domain $E=\frac{1}{L} \sum_{m=0}^{L-1} \lvert X(m) \rvert ^2$ is mostly concentrated in a compact set of frequency components, whose frequency are the ten lowest. The solid line represents the average energy for the frequency components in the UCIHAR dataset, and the shaded area indicates the range.}
    \label{fig:energy}
   \end{minipage}\hfill
   \begin{minipage}{0.46\textwidth}
     \centering
     \includegraphics[width=0.9\textwidth]{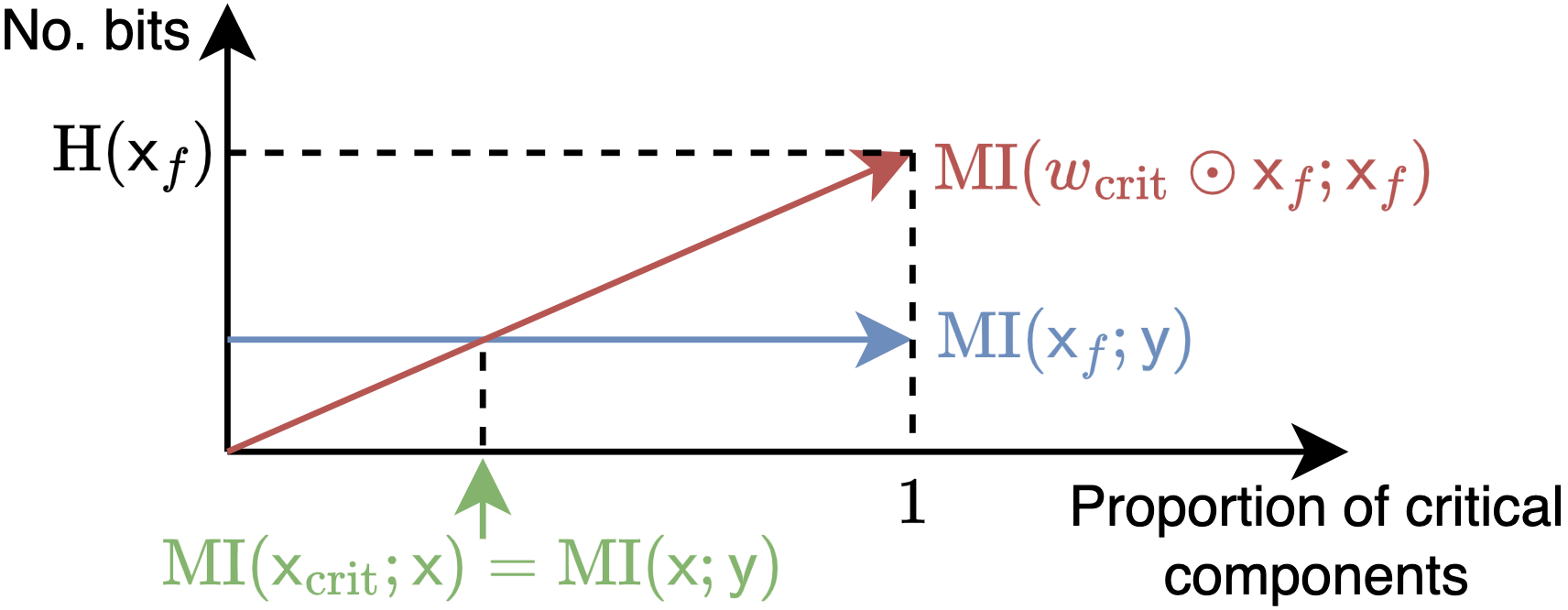}
     \caption{We aim to achieve the intersection point pointed by the green arrow where $\text{MI}(\textsf{x}_{\text{crit}}; \textsf{x}) = \text{MI}(\textsf{x}; \textsf{y})$, meaning the critical frequency components keep and only keep the semantically relevant information. The linearity of $\text{MI}(\mathbf{w}_{\text{crit}} \odot \textsf{x}_f; \textsf{x}_f)$ is for illustration purposes only.}
    \label{fig:MI-crit}
   \end{minipage}
   \vspace{-10pt}
\end{figure}

\begin{figure*}[!t]
    \centering
\includegraphics[width=0.99\textwidth]{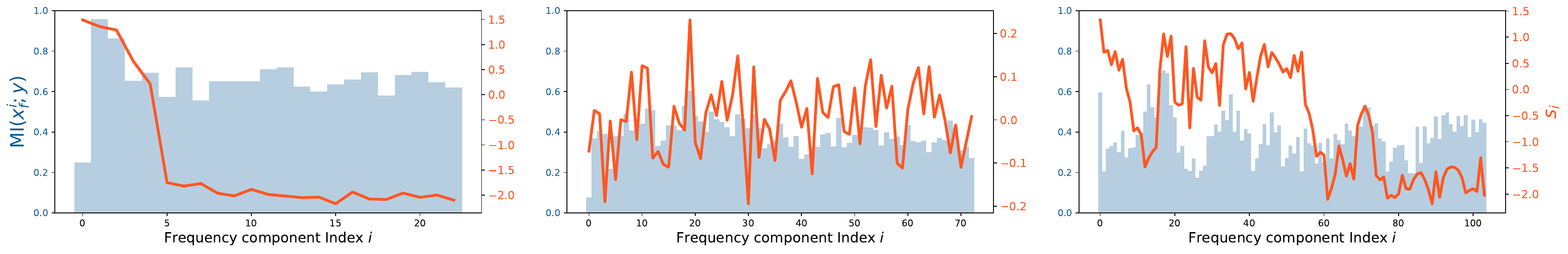}
    \caption{Despite the diverse distributions of global semantics across three datasets (Libras, ArticularyWordRecognition, and Epilepsy), as shown in the blue-grey bar plots,  the learned vector $\mathbf{s}$, represented by the orange lines, consistently captures the inherent critical information by assigning higher values to the most semantically relevant frequency components (those of high values in the bar plots).}
    \label{fig:mif_si}
\end{figure*}

\subsection{Additional Analysis for The Main Result on Time Series Classification Tasks}
\label{sec:main_analysis}
From Table~\ref{Table:result-overall}, Table~\ref{Table:ucr1}, and Table~\ref{Table:uea}, we conclude that frequency-domain augmentations outrun time-domain augmentations in general. This endorses our motivation that the frequency perspective is superior to its time-domain counterpart in preserving global semantics. 
InfoMin$^+$ exceeds other baselines on the three large HAR datasets, which demonstrates the efficacy of its objective. However, the performance gap between InfoMin$^+$ and ours indicates that directly applying an adversarial objective in the frequency domain is not customized for our approach and causes conflict in representation learning. On the other hand, our specially designed objective better suits the frequency-domain refinement. The 5 SOTA
time-series contrastive learning frameworks with carefully designed architectures and objectives become uncompetitive compared to FreRA. 

It is worth noting that our datasets cover multiple applications, diverse data scales, and various types of sensor modalities. Notably, FreRA receives the best overall performance on them, which proves that our approach provides a unified view generation approach and can be flexibly applied to various time-series applications.

\begin{figure}[!t]
    \centering
    \includegraphics[width=0.46\textwidth]{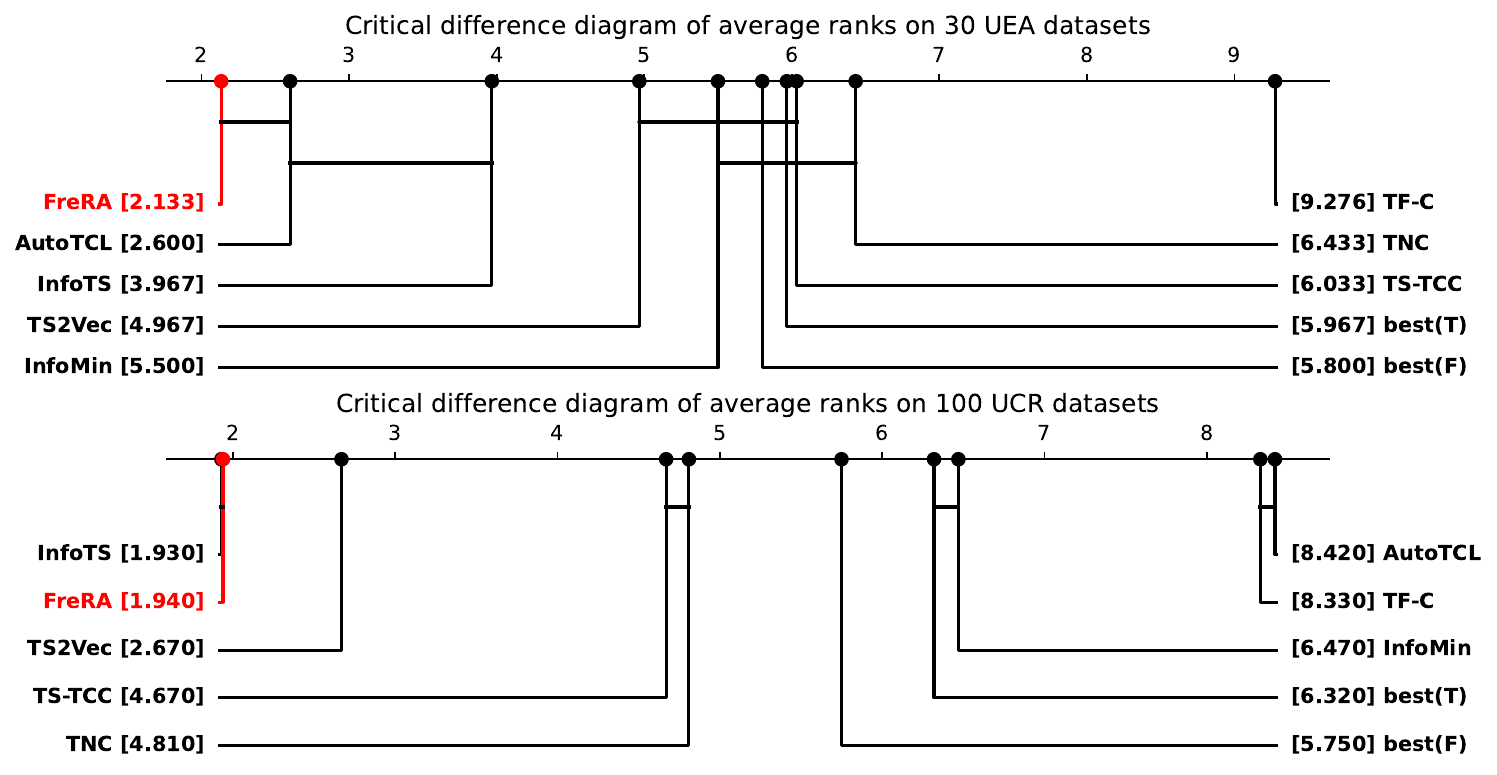}
    \caption{Critical difference diagrams on UEA and UCR archives.}
    \label{fig:cdd}
    % \vspace{-15pt}
\end{figure}

\begin{table}[!t]
      \caption{The performance of applying different statistical information to select unimportant components.}
      \label{Table:unimportant-selection}
      \centering
      % \scriptsize
      % \footnotesize
      \small
      \begin{tabular}{c|ccc}
        \toprule
        Dataset & \begin{tabular}[c]{@{}c@{}}mean\\ \textbf{(ours)}\end{tabular} & median & mean+std \\\midrule
        UICHAR  & 0.975  & 0.971  & 0.972    \\
        MS      & 0.982                                                 & 0.976  & 0.975    \\
        WISDM   & 0.972                                                 & 0.963  & 0.963   \\ \bottomrule
        \end{tabular}
% \vspace{-10pt}
\end{table}

\begin{table}[!t]
\caption{The performance of the selected sets of 5 frequency-domain augmentations on the three HAR datasets. The
best performance is highlighted in \textbf{bold}, and the second-best performance is {\ul underlined}.}
\label{Table:har-f}
\centering
\small
\begin{tabular}{ccccccc}
\toprule
Dataset &  \begin{tabular}[c]{@{}c@{}}FreRA\\  \textbf{(ours)}\end{tabular}          & lpf   & hpf   & p\_shift    & ap\_p & ap\_f       \\ \midrule
UCIHAR  & \textbf{0.975} & 0.921 & 0.939 & 0.958       & 0.959 & {\ul 0.960} \\
MS      & \textbf{0.982} & 0.934 & 0.838 & {\ul 0.970} & 0.901 & 0.952       \\
WISDM   & \textbf{0.972} & 0.934 & 0.800 & 0.943       & 0.865 & {\ul 0.950} \\
\bottomrule
\end{tabular}
\end{table}

\subsection{Ablation Studies}
\label{sec:ablation}

\noindent\textbf{Effect of Unimportant Component Selection Mechanisms.}
\label{sec:alternativeD}
To evaluate how the choice of statistical measurement in $D$ affects the final results, we conduct an ablation study comparing the performances when using mean, median, and mean+std of vector $\mathbf{s}$ as the threshold. The results are shown in Table~\ref{Table:unimportant-selection}. All the choices outperform the baseline performances in Table~\ref{Table:result-overall} and the mean value achieves the best performance among them.

% \noindent\textbf{Robustness to Gaussian Noise.}
% \label{sec:gaussian}
% In Table~\ref{Table:gaussian}, we present the performance of FreRA on the three HAR datasets in the presence of Gaussian noise. The Gaussian noise has a mean of 0 and a standard deviation of 0.8. Despite a slight degradation in performance, FreRA still outperforms all the baselines shown in Table~\ref{Table:result-overall} when there is Gaussian noise in the input time series. Note that Gaussian noise is absent in all the baselines. It demonstrates the robust performance of FreRA with respect to the Gaussian noise in the time series.

\noindent\textbf{Vector $\mathbf{s}$ Captures the Inherent Semantic Distribution in the Frequancy Domain.} 
To verify the effectiveness of FreRA, in Figure~\ref{fig:mif_si}, we visualize the learned parameter vector $\mathbf{s}$, as compared to the ground truth semantics distribution in the frequency domain, on three datasets, including Libras, ArticularyWordRecognition, and Epilepsy.  
Specifically, we use the mutual information (MI) between the frequency components with the label to quantify the ground truth importance of frequency components and presented by blue-grey bar plots. 
The distribution of important frequency components varies across datasets. The important components are distributed in low frequencies, middle frequencies, and across multiple frequencies in these datasets, respectively. The learned vector $\mathbf{s}$ which determines the importance scores of all the frequency components is presented with the orange line plots. Despite diverse distributions, $\mathbf{s}$ consistently captures the inherent critical information by learning to assign higher values to the most semantically relevant frequency components.
\label{sec:s-pres-sem}

\subsection{Additional Results}
Full results of multivariate time series classification on the UEA archive and univariate time series classification on the UCR archive are presented in Table~\ref{Table:uea} and Table~\ref{Table:ucr1}. The full result of the commonly used sets of 11 time-domain augmentations and 5 frequency-domain augmentations on the 3 HAR datasets are shown in Table~\ref{Table:har-t} and Table~\ref{Table:har-f}, respectively. 

\begin{table*}[!ht]
\caption{The overall classification result of 100 univariate time series datasets from the UCR archive. The best performance is highlighted in \textbf{bold}. }
\label{Table:ucr1}
\centering
\scriptsize
\resizebox{\textwidth}{!}{
\begin{tabular}{ccccccccccc}
\toprule
 Dataset                               & \begin{tabular}[c]{@{}c@{}}FreRA\\ \textbf{(Ours)}\end{tabular} & best(T)        & best(F)        & InfoMin        & InfoTS         & AutoTCL        & TS2Vec         & TNC            & TS-TCC         & TF-C           \\ \midrule
ACSF1                          & 0.760                                                  & 0.660          & 0.470          & 0.580          & 0.850          & 0.480          & \textbf{0.910} & 0.730          & 0.730          & 0.100          \\
AllGestureWiimoteX             & 0.707                                                  & 0.526          & 0.561          & 0.549          & 0.630          & 0.517          & \textbf{0.777} & 0.703          & 0.697          & 0.100          \\
AllGestureWiimoteY             & 0.746                                                  & 0.620          & 0.601          & 0.611          & 0.686          & 0.624          & \textbf{0.793} & 0.699          & 0.741          & 0.100          \\
AllGestureWiimoteZ             & 0.707                                                  & 0.581          & 0.573          & 0.577          & 0.629          & 0.576          & \textbf{0.770} & 0.646          & 0.689          & 0.100          \\
BeetleFly                      & \textbf{1.000}                                         & 0.700          & 0.900          & 0.800          & 0.950          & 0.650          & 0.900          & 0.850          & 0.800          & 0.450          \\
BirdChicken                    & \textbf{1.000}                                         & 0.750          & 0.850          & 0.800          & 0.900          & 0.550          & 0.800          & 0.750          & 0.650          & 0.500          \\
BME                            & \textbf{1.000}                                         & 0.920          & 0.940          & 0.967          & \textbf{1.000} & 0.640          & 0.993          & 0.973          & 0.933          & 0.630          \\
CBF                            & \textbf{1.000}                                         & 0.997          & 0.957          & 0.967          & 0.999          & 0.707          & \textbf{1.000} & 0.983          & 0.998          & 0.686          \\
Chinatown                      & \textbf{0.988}                                         & 0.910          & 0.886          & 0.939          & \textbf{0.988} & 0.983          & 0.968          & 0.977          & 0.983          & 0.904          \\
CinCECGTorso                   & \textbf{0.968}                                         & 0.912          & \textbf{0.968} & 0.914          & 0.928          & 0.305          & 0.827          & 0.669          & 0.671          & 0.248          \\
Coffee                         & \textbf{1.000}                                         & 0.964          & 0.964          & \textbf{1.000} & \textbf{1.000} & 0.750          & \textbf{1.000} & \textbf{1.000} & \textbf{1.000} & 0.464          \\
Computers                      & \textbf{0.776}                                         & 0.684          & 0.700          & 0.676          & 0.748          & 0.468          & 0.660          & 0.684          & 0.704          & 0.644          \\
Crop                           & 0.755                                                  & 0.569          & 0.566          & 0.561          & \textbf{0.766} & 0.608          & 0.756          & 0.738          & 0.742          & 0.632          \\
DiatomSizeReduction            & 0.980                                                  & 0.817          & 0.948          & 0.905          & \textbf{0.997} & 0.676          & 0.987          & 0.993          & 0.977          & 0.301          \\
DistalPhalanxOutlineAgeGroup   & \textbf{0.799}                                         & 0.655          & 0.755          & 0.645          & 0.763          & 0.640          & 0.727          & 0.741          & 0.755          & 0.732          \\
DistalPhalanxOutlineCorrect    & \textbf{0.804}                                         & 0.627          & 0.670          & 0.734          & 0.801          & 0.583          & 0.775          & 0.754          & 0.754          & 0.683          \\
DistalPhalanxTW                & \textbf{0.755}                                         & 0.676          & 0.719          & 0.669          & 0.727          & 0.597          & 0.698          & 0.669          & 0.676          & 0.669          \\
DodgerLoopDay                  & 0.600                                                  & 0.275          & 0.325          & 0.388          & \textbf{0.675} & 0.338          & 0.562          & -              & -              & 0.150          \\
DodgerLoopGame                 & \textbf{0.942}                                         & 0.833          & 0.797          & 0.855          & \textbf{0.942} & 0.725          & 0.841          & -              & -              & 0.522          \\
DodgerLoopWeeken               & \textbf{0.993}                                         & 0.935          & 0.949          & 0.978          & 0.986          & 0.920          & 0.964          & -              & -              & 0.739          \\
Earthquakes                    & 0.820                                                  & 0.748          & 0.748          & 0.748          & \textbf{0.821} & 0.748          & 0.748          & 0.748          & 0.748          & 0.748          \\
ECG200                         & 0.890                                                  & 0.830          & 0.840          & 0.800          & 0.930          & 0.700          & 0.920          & 0.830          & 0.880          & \textbf{0.940} \\
ECG5000                        & \textbf{0.948}                                         & 0.935          & 0.940          & 0.926          & 0.945          & 0.900          & 0.935          & 0.937          & 0.941          & 0.938          \\
ECGFiveDays                    & \textbf{1.000}                                         & 0.998          & 0.987          & 0.990          & \textbf{1.000} & 0.821          & \textbf{1.000} & 0.999          & 0.878          & 0.972          \\
ElectricDevices                & 0.657                                                  & 0.609          & 0.599          & 0.609          & 0.702          & 0.562          & \textbf{0.721} & 0.700          & 0.686          & 0.560          \\
EOGHorizontalSignal            & \textbf{0.597}                                         & 0.434          & 0.508          & 0.470          & 0.572          & 0.293          & 0.544          & 0.442          & 0.401          & 0.083          \\
EOGVerticalSignal              & 0.489                                                  & 0.320          & 0.423          & 0.312          & 0.459          & 0.290          & \textbf{0.503} & 0.392          & 0.376          & 0.144          \\
FaceAll                        & 0.888                                                  & 0.628          & 0.728          & 0.658          & \textbf{0.929} & 0.689          & 0.805          & 0.766          & 0.813          & 0.714          \\
FaceFour                       & 0.864                                                  & 0.773          & 0.773          & 0.773          & 0.818          & 0.205          & \textbf{0.932} & 0.659          & 0.773          & 0.330          \\
FacesUCR                       & 0.866                                                  & 0.861          & 0.794          & 0.760          & 0.913          & 0.544          & \textbf{0.930} & 0.789          & 0.863          & 0.779          \\
FordA                          & 0.943                                                  & 0.905          & 0.902          & 0.917          & 0.915          & 0.494          & \textbf{0.948} & 0.902          & 0.930          & 0.537          \\
FordB                          & \textbf{0.832}                                         & 0.775          & 0.794          & 0.780          & 0.785          & 0.493          & 0.807          & 0.733          & 0.815          & 0.474          \\
FreezerRegularTrain            & 0.994                                                  & 0.804          & 0.856          & 0.820          & \textbf{0.996} & 0.717          & 0.986          & 0.991          & 0.989          & 0.742          \\
FreezerSmallTrain              & \textbf{0.988}                                         & 0.787          & 0.735          & 0.811          & \textbf{0.988} & 0.721          & 0.894          & 0.982          & 0.979          & 0.501          \\
Fungi                          & 0.941                                                  & 0.667          & 0.667          & 0.677          & 0.946          & 0.263          & \textbf{0.962} & 0.527          & 0.753          & 0.860          \\
GestureMidAirD1                & \textbf{0.638}                                         & 0.315          & 0.431          & 0.308          & 0.592          & 0.423          & 0.631          & 0.431          & 0.369          & 0.038          \\
GestureMidAirD2                & \textbf{0.608}                                         & 0.292          & 0.138          & 0.269          & 0.492          & 0.177          & 0.515          & 0.362          & 0.254          & 0.038          \\
GesturePebbleZ1                & 0.779                                                  & 0.581          & 0.610          & 0.506          & 0.802          & 0.650          & \textbf{0.930} & 0.378          & 0.395          & 0.163          \\
GesturePebbleZ2                & 0.722                                                  & 0.614          & 0.551          & 0.475          & 0.842          & 0.424          & \textbf{0.873} & 0.316          & 0.430          & 0.152          \\
GunPoint                       & \textbf{1.000}                                         & 0.887          & 0.993          & 0.933          & \textbf{1.000} & 0.800          & 0.987          & 0.967          & 0.993          & 0.573          \\
GunPointAgeSpan                & 0.984                                                  & 0.921          & 0.908          & 0.908          & \textbf{1.000} & 0.639          & 0.994          & 0.984          & 0.994          & 0.927          \\
GunPointMaleVersusFemale       & \textbf{1.000}                                         & 0.835          & 0.839          & 0.832          & \textbf{1.000} & 0.718          & \textbf{1.000} & 0.994          & 0.997          & 0.987          \\
GunPointOldVersusYoung         & \textbf{1.000}                                         & \textbf{1.000} & \textbf{1.000} & \textbf{1.000} & \textbf{1.000} & 0.981          & \textbf{1.000} & \textbf{1.000} & \textbf{1.000} & \textbf{1.000} \\
Ham                            & 0.810                                                  & 0.790          & 0.705          & 0.686          & \textbf{0.838} & 0.533          & 0.724          & 0.752          & 0.743          & 0.752          \\
HandOutlines                   & 0.900                                                  & 0.876          & 0.886          & 0.881          & \textbf{0.946} & 0.662          & 0.930          & 0.930          & 0.724          & 0.641          \\
Haptics                        & 0.487                                                  & 0.471          & 0.487          & 0.406          & \textbf{0.546} & 0.334          & 0.536          & 0.474          & 0.396          & 0.208          \\
Herring                        & \textbf{0.703}                                         & 0.594          & 0.609          & 0.594          & 0.656          & 0.594          & 0.641          & 0.594          & 0.594          & 0.594          \\
HouseTwenty                    & \textbf{0.983}                                         & 0.891          & 0.706          & 0.681          & 0.924          & 0.655          & 0.941          & 0.782          & 0.790          & 0.571          \\
InlineSkate                    & 0.353                                                  & 0.258          & 0.318          & 0.242          & \textbf{0.424} & 0.193          & 0.415          & 0.378          & 0.347          & 0.155          \\
InsectEPGRegularTrain          & \textbf{1.000}                                         & \textbf{1.000} & \textbf{1.000} & \textbf{1.000} & \textbf{1.000} & \textbf{1.000} & \textbf{1.000} & \textbf{1.000} & \textbf{1.000} & \textbf{1.000} \\
InsectEPGSmallTrain            & \textbf{1.000}                                         & \textbf{1.000} & 0.451          & \textbf{1.000} & \textbf{1.000} & \textbf{1.000} & \textbf{1.000} & \textbf{1.000} & \textbf{1.000} & 0.474          \\
ItalyPowerDemand               & \textbf{0.976}                                         & 0.968          & 0.956          & 0.969          & 0.966          & 0.614          & 0.961          & 0.928          & 0.955          & 0.934          \\
LargeKitchenAppliances         & 0.848                                                  & 0.787          & 0.784          & 0.776          & 0.853          & 0.416          & \textbf{0.875} & 0.776          & 0.848          & 0.389          \\
Lightning2                     & \textbf{0.951}                                         & 0.672          & 0.721          & 0.721          & 0.934          & 0.639          & 0.869          & 0.869          & 0.836          & 0.738          \\
Lightning7                     & 0.840                                                  & 0.726          & 0.767          & 0.712          & \textbf{0.877} & 0.342          & 0.863          & 0.767          & 0.685          & 0.616          \\
Mallat                         & 0.954                                                  & 0.820          & 0.907          & 0.722          & \textbf{0.974} & 0.412          & 0.915          & 0.871          & 0.922          & 0.123          \\
Meat                           & 0.917                                                  & 0.333          & 0.774          & 0.333          & \textbf{0.967} & 0.583          & \textbf{0.967} & 0.917          & 0.883          & 0.333          \\
MiddlePhalanxOutlineAgeGroup   & \textbf{0.669}                                         & 0.610          & 0.617          & 0.597          & 0.662          & 0.577          & 0.636          & 0.643          & 0.630          & 0.578          \\
MiddlePhalanxOutlineCorrect    & 0.842                                                  & 0.570          & 0.704          & 0.570          & \textbf{0.859} & 0.500          & 0.838          & 0.818          & 0.818          & 0.653          \\
MiddlePhalanxTW                & \textbf{0.630}                                         & 0.558          & 0.578          & 0.571          & 0.617          & 0.552          & 0.591          & 0.571          & 0.610          & 0.558          \\
\end{tabular}}
\end{table*}

\begin{table*}[!ht]
\label{Table:ucr2}
\centering
\scriptsize
\resizebox{\textwidth}{!}{
\begin{tabular}{ccccccccccc}
\toprule
 Dataset                               & \begin{tabular}[c]{@{}c@{}}FreRA\\ \textbf{(Ours)}\end{tabular} & best(T)        & best(F)        & InfoMin        & InfoTS         & AutoTCL        & TS2Vec         & TNC            & TS-TCC         & TF-C           \\ \midrule

MixedShapesRegularTrain        & 0.925                                                  & 0.878          & 0.927          & 0.829          & \textbf{0.935} & 0.624          & 0.922          & 0.911          & 0.855          & 0.400          \\
MixedShapesSmallTrain          & 0.852                                                  & 0.822          & 0.842          & 0.776          & \textbf{0.887} & 0.525          & 0.881          & 0.813          & 0.735          & 0.181          \\
MoteStrain                     & 0.891                                                  & \textbf{0.904} & 0.806          & 0.849          & 0.873          & 0.676          & 0.863          & 0.825          & 0.843          & 0.815          \\
OliveOil                       & 0.800                                                  & 0.400          & 0.752          & 0.400          & \textbf{0.933} & 0.600          & 0.900          & 0.833          & 0.800          & 0.400          \\
OSULeaf                        & \textbf{0.909}                                         & 0.678          & 0.705          & 0.554          & 0.760          & 0.384          & 0.876          & 0.723          & 0.723          & 0.467          \\
Phoneme                        & 0.273                                                  & 0.211          & 0.200          & 0.208          & 0.281          & 0.158          & \textbf{0.312} & 0.180          & 0.242          & 0.104          \\
PickupGestureWimoteZ           & \textbf{0.860}                                         & 0.740          & 0.399          & 0.680          & 0.820          & 0.640          & 0.820          & 0.620          & 0.600          & 0.100          \\
PigCVP                         & 0.611                                                  & 0.303          & 0.667          & 0.207          & 0.653          & 0.130          & \textbf{0.870} & 0.649          & 0.615          & 0.019          \\
PLAID                          & 0.523                                                  & 0.330          & 0.269          & 0.307          & 0.355          & 0.451          & \textbf{0.561} & 0.495          & 0.445          & 0.061          \\
Plane                          & \textbf{1.000}                                         & \textbf{1.000} & \textbf{1.000} & \textbf{1.000} & \textbf{1.000} & \textbf{1.000} & \textbf{1.000} & \textbf{1.000} & \textbf{1.000} & 0.952          \\
PowerCons                      & 0.994                                                  & 0.939          & 0.933          & 0.917          & \textbf{1.000} & 0.861          & 0.972          & 0.933          & 0.961          & 0.894          \\
ProximalPhalanxOutlineAgeGroup & \textbf{0.888}                                         & 0.854          & 0.883          & 0.873          & 0.883          & 0.715          & 0.844          & 0.854          & 0.839          & 0.849          \\
ProximalPhalanxOutlineCorrect  & 0.893                                                  & 0.722          & 0.784          & 0.698          & \textbf{0.927} & 0.820          & 0.900          & 0.866          & 0.873          & 0.801          \\
ProximalPhalanxTW              & \textbf{0.849}                                         & 0.678          & 0.800          & 0.780          & 0.844          & 0.771          & 0.824          & 0.810          & 0.800          & 0.795          \\
RefrigerationDevices           & 0.597                                                  & 0.549          & 0.533          & 0.501          & \textbf{0.624} & 0.360          & 0.589          & 0.565          & 0.563          & 0.299          \\
Rock                           & 0.700                                                  & 0.480          & 0.480          & 0.500          & \textbf{0.760} & 0.400          & 0.700          & 0.580          & 0.600          & 0.280          \\
ScreenType                     & 0.491                                                  & 0.368          & \textbf{0.656} & 0.421          & 0.493          & 0.355          & 0.411          & 0.509          & 0.419          & 0.344          \\
ShakeGestureWiimoteZ           & \textbf{0.980}                                         & 0.840          & 0.800          & 0.840          & 0.920          & 0.787          & 0.940          & 0.820          & 0.860          & 0.100          \\
ShapeletSim                    & \textbf{1.000}                                         & \textbf{1.000} & 0.978          & \textbf{1.000} & 0.856          & 0.533          & \textbf{1.000} & 0.589          & 0.683          & 0.467          \\
ShapesAll                      & 0.822                                                  & 0.368          & 0.627          & 0.415          & 0.852          & 0.802          & \textbf{0.905} & 0.788          & 0.773          & 0.582          \\
SmoothSubspace                 & 0.987                                                  & 0.873          & 0.893          & 0.860          & \textbf{1.000} & 0.913          & 0.993          & 0.913          & 0.953          & 0.653          \\
SonyAIBORobotSurface2          & \textbf{0.957}                                         & 0.815          & 0.867          & 0.807          & 0.953          & 0.769          & 0.890          & 0.834          & 0.907          & 0.846          \\
SonyAIBORobotSurfacel          & \textbf{0.953}                                         & 0.885          & 0.906          & 0.854          & 0.927          & 0.778          & 0.903          & 0.804          & 0.899          & 0.804          \\
StarLightCurves                & \textbf{0.973}                                         & 0.874          & 0.964          & 0.891          & \textbf{0.973} & 0.849          & 0.971          & 0.968          & 0.967          & 0.855          \\
Strawberry                     & 0.965                                                  & 0.835          & 0.876          & 0.849          & \textbf{0.978} & 0.614          & 0.965          & 0.951          & 0.965          & 0.832          \\
SwedishLeaf                    & \textbf{0.950}                                         & 0.789          & 0.874          & 0.787          & \textbf{0.950} & 0.794          & 0.942          & 0.880          & 0.923          & 0.891          \\
Symbols                        & \textbf{0.980}                                         & 0.912          & 0.943          & 0.847          & 0.979          & 0.699          & 0.976          & 0.885          & 0.916          & 0.174          \\
SyntheticControl               & \textbf{1.000}                                         & 0.990          & 0.980          & 0.997          & \textbf{1.000} & 0.880          & 0.997          & \textbf{1.000} & 0.990          & 0.760          \\
ToeSegmentation1               & \textbf{0.961}                                         & 0.943          & 0.921          & 0.939          & 0.934          & 0.496          & 0.947          & 0.864          & 0.930          & 0.570          \\
ToeSegmentation2               & \textbf{0.931}                                         & 0.815          & 0.800          & 0.869          & 0.915          & 0.692          & 0.915          & 0.831          & 0.877          & 0.338          \\
Trace                          & \textbf{1.000}                                         & 0.880          & \textbf{1.000} & 0.920          & \textbf{1.000} & 0.650          & \textbf{1.000} & \textbf{1.000} & \textbf{1.000} & 0.690          \\
TwoLeadECG                     & 0.987                                                  & 0.867          & 0.984          & 0.901          & \textbf{0.998} & 0.565          & 0.987          & 0.993          & 0.976          & 0.921          \\
TwoPatterns                    & \textbf{1.000}                                         & 0.997          & 0.999          & 0.957          & \textbf{1.000} & 0.264          & \textbf{1.000} & \textbf{1.000} & 0.999          & 0.654          \\
UMD                            & \textbf{1.000}                                         & 0.938          & 0.617          & 0.979          & \textbf{1.000} & 0.590          & \textbf{1.000} & 0.993          & 0.986          & 0.778          \\
Wafer                          & 0.996                                                  & 0.960          & 0.956          & 0.959          & \textbf{0.998} & 0.921          & \textbf{0.998} & 0.994          & 0.994          & 0.994          \\
Wine                           & 0.833                                                  & 0.500          & 0.500          & 0.500          & \textbf{0.963} & 0.500          & 0.889          & 0.759          & 0.778          & 0.500          \\
WordSynonyms                   & 0.619                                                  & 0.350          & 0.384          & 0.359          & \textbf{0.704} & 0.497          & \textbf{0.704} & 0.630          & 0.531          & 0.487          \\
Worms                          & \textbf{0.792}                                         & 0.558          & 0.636          & 0.623          & 0.753          & 0.403          & 0.701          & 0.623          & 0.753          & 0.429          \\
WormsTwoClass                  & 0.831                                                  & 0.753          & 0.714          & 0.714          & \textbf{0.857} & 0.558          & 0.805          & 0.727          & 0.753          & 0.584          \\
Yoga                           & 0.808                                                  & 0.693          & 0.699          & 0.607          & 0.869          & 0.536          & \textbf{0.887} & 0.812          & 0.791          & 0.688          \\ \midrule
Avg. ACC                       & \textbf{0.850}                                         & 0.723          & 0.744          & 0.718          & 0.849          & 0.598          & 0.845          & 0.776          & 0.780          & 0.542          \\
Avg. RANK                      & 1.940                                                  & 6.320          & 5.750          & 6.470          & \textbf{1.930} & 8.420          & 2.670          & 4.810          & 4.670          & 8.330    \\
\bottomrule
\end{tabular}}
\end{table*}

\begin{table*}[!t]
\caption{The overall classification result of 30 multivariate time series datasets from the UEA archive. The best performance is highlighted in \textbf{bold}. }
\label{Table:uea}
\centering
\scriptsize
\resizebox{\textwidth}{!}{
\begin{tabular}{ccccccccccc}
\toprule
Dataset                    & \begin{tabular}[c]{@{}c@{}}FreRA\\  \textbf{(ours)}\end{tabular}          & best(T)        & best(F)        & InfoMin        & InfoTS         & AutoTCL        & TS2Vec         & TNC            & TS-TCC         & TF-C  \\ \midrule
Articulary WordRecognition & \textbf{0.990} & 0.887          & 0.947          & 0.913          & 0.987          & 0.983          & 0.987          & 0.973          & 0.953          & 0.467 \\
AtrialFibrillation         & \textbf{0.467} & 0.400          & 0.333          & 0.267          & 0.200          & \textbf{0.467} & 0.200          & 0.133          & 0.267          & 0.040 \\
BasicMotions               & \textbf{1.000} & \textbf{1.000} & \textbf{1.000} & \textbf{1.000} & 0.975          & \textbf{1.000} & 0.975          & 0.975          & \textbf{1.000} & 0.475 \\
CharacterTrajectories      & 0.991          & 0.953          & 0.976          & 0.990          & 0.974          & 0.976          & \textbf{0.995} & 0.967          & 0.985          & 0.090 \\
Cricket                    & \textbf{1.000} & 0.986          & 0.986          & 0.958          & 0.986          & \textbf{1.000} & 0.972          & 0.958          & 0.917          & 0.125 \\
DuckDuckGeese              & \textbf{0.760} & 0.660          & 0.660          & 0.700          & 0.540          & 0.700          & 0.680          & 0.460          & 0.380          & 0.340 \\
Eigen Worms                & 0.863          & 0.779          & 0.840          & 0.794          & 0.733          & \textbf{0.901} & 0.847          & 0.840          & 0.779          & -     \\
Epilepsy                   & \textbf{0.993} & 0.906          & 0.935          & 0.920          & 0.971          & 0.978          & 0.964          & 0.957          & 0.957          & 0.217 \\
ERing                      & 0.919          & 0.885          & 0.907          & 0.904          & \textbf{0.949} & 0.944          & 0.874          & 0.852          & 0.904          & 0.167 \\
EthanolConcentration       & 0.323          & 0.297          & 0.262          & 0.243          & 0.281          & \textbf{0.354} & 0.308          & 0.297          & 0.285          & 0.247 \\
FaceDetection              & \textbf{0.581} & 0.564          & 0.521          & 0.560          & 0.534          & \textbf{0.581} & 0.501          & 0.536          & 0.544          & 0.502 \\
FingerMovements            & 0.610          & 0.530          & 0.500          & 0.500          & 0.630          & \textbf{0.640} & 0.480          & 0.470          & 0.460          & 0.510 \\
HandMovementDirection      & \textbf{0.514} & 0.378          & 0.365          & 0.324          & 0.392 & 0.432          & 0.338          & 0.324          & 0.243          & 0.405 \\
Handwriting                & \textbf{0.593} & 0.501          & 0.469          & 0.569          & 0.452          & 0.384          & 0.515          & 0.249          & 0.498          & 0.051 \\
Heartbeat                  & \textbf{0.785} & 0.741          & 0.746          & 0.737          & 0.722          & \textbf{0.785} & 0.683          & 0.746          & 0.751          & 0.737 \\
Japanese Vowels            & 0.965          & 0.938          & 0.938          & 0.938          & \textbf{0.984} & \textbf{0.984} & \textbf{0.984} & 0.978          & 0.930          & 0.135 \\
Libras                     & \textbf{0.911} & 0.761          & 0.822          & 0.800          & 0.883          & 0.833          & 0.867          & 0.817          & 0.822          & 0.067 \\
LSST                       & 0.494          & 0.393          & 0.391          & 0.473          & 0.591          & 0.554          & 0.537          & \textbf{0.595} & 0.474          & 0.314 \\
MotorImagery               & 0.550          & 0.530          & 0.540          & 0.530          & \textbf{0.630} & 0.570          & 0.510          & 0.500          & 0.610          & 0.500 \\
NATOPS                     & 0.900          & 0.867          & 0.872          & 0.822          & 0.933          & \textbf{0.944} & 0.928          & 0.911          & 0.822          & 0.533 \\
PEMS-SF                    & 0.746          & 0.653          & 0.671          & 0.699          & 0.751          & \textbf{0.838} & 0.682          & 0.699          & 0.734          & 0.312 \\
PenDigits                  & 0.973          & 0.946          & 0.946          & 0.970          & \textbf{0.990} & 0.984          & 0.989          & 0.979          & 0.974          & 0.236 \\
PhonemeSpectra             & \textbf{0.274} & 0.226          & 0.226          & 0.240          & 0.249          & 0.218          & 0.233          & 0.207          & 0.252          & 0.026 \\
RacketSports               & 0.888          & 0.816          & 0.796          & 0.822          & 0.855          & \textbf{0.914} & 0.855          & 0.776          & 0.816          & 0.480 \\
SelfRegulationSCP1         & \textbf{0.908} & 0.836          & 0.870          & 0.867          & 0.874          & 0.891          & 0.812          & 0.799          & 0.823          & 0.502 \\
SelfRegulationSCP2         & \textbf{0.622} & 0.589          & 0.594          & \textbf{0.622} & 0.578          & 0.578          & 0.578          & 0.550          & 0.533          & 0.500 \\
SpokenArabicDigits         & \textbf{0.984} & 0.935          & 0.871          & 0.981          & 0.947          & 0.925          & 0.932          & 0.934          & 0.970          & 0.100 \\
StandWalkJump              & \textbf{0.667} & 0.400          & 0.333          & 0.333          & 0.467          & 0.533          & 0.467          & 0.400          & 0.333          & 0.333 \\
UWaveGestureLibrary        & \textbf{0.900} & 0.794          & 0.800          & 0.872          & 0.884          & 0.893          & 0.884          & 0.759          & 0.753          & 0.125 \\
InsectWingbeat             & 0.462          & 0.363          & 0.456          & 0.443          & 0.470          & \textbf{0.488} & 0.466          & 0.469          & 0.264          & 0.108 \\
\midrule
Avg. ACC                   & \textbf{0.754} & 0.684          & 0.686          & 0.693          & 0.714          & 0.742          & 0.704          & 0.670          & 0.668          & 0.298 \\
Avg. RANK                  & \textbf{2.133} & 5.967          & 5.800          & 5.500          & 3.967          & 2.600          & 4.967          & 6.433          & 6.033          & 9.276 \\ \bottomrule
\end{tabular}}
\end{table*}

\begin{table*}[!t]
\caption{The performance of the selected sets of 11 time-domain augmentations on the three HAR datasets. The
best performance is highlighted in \textbf{bold}, and the second-best performance is {\ul underlined}. `t\_flip’, `t\_warp’, `perm\_jit’ and `jit\_scal’ are short
for \texttt{time-flipping}, \texttt{time-warping}, \texttt{permutation-and-jitter} and \texttt{jitter-and-scale}.}
\label{Table:har-t}
\centering
\scriptsize
\resizebox{\textwidth}{!}{
\begin{tabular}{ccccccccccccc}
\toprule
Dataset & \begin{tabular}[c]{@{}c@{}}FreRA\\  \textbf{(ours)}\end{tabular}          & jit         & scale & negation & perm  & shuffling & t\_flip & t\_warp & resample    & rotation & perm\_jit   & jit\_scal \\ \midrule
UCIHAR  & \textbf{0.975} & 0.958       & 0.940 & 0.892    & 0.910 & 0.913     & 0.917   & 0.934   & 0.947       & 0.596    & {\ul 0.959} & 0.945     \\
MS      & \textbf{0.982} & 0.930       & 0.914 & 0.813    & 0.927 & 0.910     & 0.915   & 0.925   & {\ul 0.956} & 0.887    & 0.948       & 0.915     \\
WISDM   & \textbf{0.972} & {\ul 0.942} & 0.928 & 0.901    & 0.932 & 0.925     & 0.884   & 0.910   & {\ul 0.942} & 0.872    & 0.932       & 0.927   \\
\bottomrule
\end{tabular}}
\end{table*}

\end{document}